\documentclass[nohyperref]{article}

\usepackage{aux/amirhesam_macros}
\usepackage{microtype}
\usepackage{graphicx}
\usepackage{subfigure}
\usepackage{booktabs} % for professional tables

\usepackage{hyperref}

\usepackage[accepted]{icml2023}

\usepackage{amsmath}
\usepackage{amssymb}
\usepackage{mathtools}
\usepackage{amsthm}

\usepackage[capitalize,noabbrev]{cleveref}

\usepackage{gensymb}

\theoremstyle{plain}
\newtheorem{theorem}{Theorem}[section]
\newtheorem{proposition}[theorem]{Proposition}
\newtheorem{lemma}[theorem]{Lemma}
\newtheorem{corollary}[theorem]{Corollary}
\theoremstyle{definition}

\newtheorem{assumption}[theorem]{Assumption}
\newtheorem{remark}[theorem]{Remark}

\def\rva{{\mathbf{a}}}
\def\rvb{{\mathbf{b}}}

\def\rvv{{\mathbf{v}}}
\def\rvw{{\mathbf{w}}}
\def\rvx{{\mathbf{x}}}

\def\rvz{{\mathbf{z}}}

\usepackage[textsize=tiny]{todonotes}

\icmltitlerunning{On Emergence of Clean-Priority Learning  in Early Stopped Neural Networks}

\begin{document}

\twocolumn[
\icmltitle{On Emergence of Clean-Priority Learning  in Early Stopped Neural Networks}



\icmlsetsymbol{equal}{*}

\begin{icmlauthorlist}
\icmlauthor{Chaoyue Liu}{equal,hdsi}
\icmlauthor{Amirhesam Abedsoltan}{equal,cse}
\icmlauthor{Mikhail Belkin}{hdsi}
\end{icmlauthorlist}

\icmlaffiliation{cse}{Department of Computer Science and Engineering, and}

\icmlaffiliation{hdsi}{Halicioglu Data Science Institute, UC San Diego, USA}

\icmlcorrespondingauthor{Chaoyue Liu, Amirhesam Abedsoltan}{chl212,aabedsoltan@ucsd.edu}

\icmlkeywords{Early stopping, dynamics, label noisy}

\vskip 0.3in
]


% \icmlEqualContribution
\printAffiliationsAndNotice{\icmlEqualContribution} % otherwise use the standard text.

\begin{abstract}
When random label noise is added to a training dataset, the prediction error of a neural network on a label-noise-free test dataset initially improves during early training but eventually deteriorates, following a U-shaped dependence on training time. This behaviour is believed to be a result of neural networks learning the pattern of clean data first and fitting the noise later in the training, a phenomenon that we refer to as \textit{clean-priority learning}. In this study, we aim to explore the learning dynamics underlying this phenomenon. We theoretically demonstrate that,  in the early stage of training, the update direction of gradient descent is determined by the clean subset of training data, leaving the noisy subset  has minimal to no impact, resulting in a prioritization of clean learning. Moreover, we show both theoretically and experimentally, as the clean-priority learning goes on, the dominance of the  gradients of clean samples over those of noisy samples  diminishes, and finally results in a termination of the clean-priority learning and fitting of the noisy samples. 
\end{abstract}

\section{Introduction}

Recent studies suggest that Neural Network (NN) models tend to first learn the patterns in the clean data and overfit the noise at a later stage \cite{arpit2017closer,li2020gradient}. We refer to this phenomena as \textit{clean-priority learning}.
Real datasets may have intrinsic label noise, which is why early stopping can be useful in practice, saving a significant amount of unnecessary computation cost. 

To study \textit{clean-priority learning} phenomena, 
we intentionally add label noise to the training data set and leave the test data set untouched, this is a common setting in literature (see, for example, \cite{zhang2021understanding,nakkiran2021deep, belkin2018understand}).
Figure \ref{fig:intro_illu} illustrate this for a MNIST classification task using CNN. 
The test prediction error exhibits a U-shaped dependence on training time, with an initial decrease followed by an increase after the early stopping point. The  observation is that in the intermediate steps, especially around the early stopping point, the test performance can be significantly better than the label noise level added to the training set (below the dashed line).

\begin{figure}[t]
        \centering        
        \includegraphics[width=0.92\columnwidth]
        {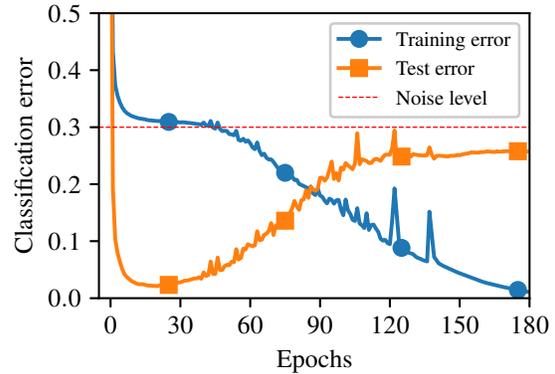}%
\label{fig:intro_illu}
\vspace{-10pt}
        \caption{Classification errors on training and test datasets of MNIST using CNN. Test error exhibits a U-shaped curve, and can be significantly lower than the noise level during training.
        }
\vspace{-10pt}
\end{figure}

To further explore this phenomenon, we address the following fundamental questions:
\begin{enumerate}
    \item {\it What is the underlying mechanism by which neural networks learn the clean data first and fit the noise in later stages?}
    \item {\it How does the model performance deteriorate after the early stopping point ?}
\end{enumerate}

At the outset, we analyze the configuration of sample-wise gradients (or its variant, for multi-class classification) on the training dataset, at the neural network initialization. Our objective is to examine whether there exists any pattern among the gradients that can explain the clean-priority learning phenomenon. Our analysis reveals that, at initialization, samples within the same class (before label corruption), which are presumably more similar to each other, tend to have their sample-wise gradients relatively closer in vector directions (compared to the samples from different classes). The label corruption, which flips the label to a different class, flips the corresponding sample-wise gradient to its opposite direction. Consequently, the sum of the noisy sample gradients is in sharp opposite direction of that of the clean sample gradients.

The key observation is that, due to the dominance of the population of the clean samples, in the early stage of learning, the gradient of noisy subset is cancelled out, and essentially makes no contribution on the gradient descent (GD) update direction\footnote{To be more precise, the only effect of the noisy subset gradient is resulting in a smaller GD step size.}. It is also worth noting that almost all clean sample-wise gradient vectors ``agree'' with the GD update (i.e., have positive projection), while almost all noisy sample-wise gradients are ``against'' the GD update. As a result, the individual loss on each clean sample is decreased, and that on each noisy sample is increased.
Hence, we see that, in the early stage, the GD algorithm is determined by the clean samples and exhibits the clean-priority learning. 
 
We further show that as the clean-priority learning process continues, the clean subset gradient's dominance over the noisy subset gradually diminishes. This is particularly evident around the early stopping point, at which the noisy subset gradient begins to make a meaningful contribution, causing the model to fit the noisy samples along with the clean ones. This new trend in learning behavior is expected to hurt the model's performance, which was previously based primarily on the clean samples.

In summary, we make the following contributions:
\begin{itemize}
    \vspace{-5pt}
    \item {\bf learning dynamics.} In the early stage of learning of neural networks, the noisy samples contribution to the GD update is cancelled out by that of the clean samples, which is the key mechanism underlying  clean-priority learning. 
    However, this clean-priority learning behavior gradually fades as the dominance of the clean subset diminishes, particularly around the early stopping point. We experimentally verify our findings on deep neural networks on various classification problems.
    \item For fully connected networks with mild assumption on data we theoretically prove our empirical observation. 
    \item In addition, we find for neural networks, at initialization, sample-wise gradients from the same class tend to have relatively similar directions, when there is no label noise.
    
    \vspace{-5pt}
\end{itemize}

The paper is organized as follows: in Section \ref{sec:setup}, we describe the setup of the problems and introduce necessary concepts and notations. In Section \ref{sec:intialization}, we analyze the sample-wise gradients at initialization of neural network, for binary classification. In Section \ref{sec:binary_dynamics}, we show the learning dynamics, especially the clean-priority learning, on binary classification. In Section \ref{sec:multi_class}, we extend our study and findings to multi-class classification problems.

\subsection{Related works}

Early stopping is often considered as a regularization technique and is widely used in practice to obtain good performance for machine learning models \cite{zhang2017sensitivity,gal2016theoretically,graves2013speech}.
Early stopping also received a lot theoretical analyses, both on non-neural network models, especially linear regression and kernel regression \cite{yao2007early, ali2019continuous,xu2022relaxing,shen2022optimal}, and on neural networks \cite{zhang2021understanding,ji2021early}.
% \misha{this sentence is not grammatical, also not particularly informative. Why emphasize generalization bounds?}

Recent studies suggest that, when random label noise presents, neural networks fit the clean data first and ``overfit'' the noise later on \cite{arpit2017closer, li2020gradient, bai2021understanding}. For example, based on experimental observations that maximum validation set accuracy is achieved before good training set accuracy, the work \cite{arpit2017closer} conjectures that the neural network  learns clean patterns first. However, there is no explanation on why and how the clean-priority learning happens. In another work \cite{li2020gradient}, assuming (almost) perfectly cluster-able data and uniform conditioning on Jacobian matrices, proves that clean data are fit by two-layer neural networks in an early stage. However, this data assumption requires, at the same location of each noisy sample, there must exist several (at least $1/\delta$, with $\delta$ being the label noise level) clean samples. This assumption is often not met by actual datasets. 

To the best of our knowledge, our work is the first to elucidate the mechanism underlying the clean-priority phenomenon, offering a fresh perspective on its dynamics.

\section{ Problem setup and preliminary }
\label{sec:setup}
In this paper, we consider supervised classification problems. 

\paragraph{Datasets.} 
There is a training dataset $\mathcal{D} \triangleq \{(\rvx_i, y_i)\}_{i=1}^n$ of size $|\mathcal{D}|=n$.
In each sample $(\rvx_i, y_i)$, there are input features $\rvx_i\in\mathbb{R}^d$ and a label $y_i$. For binary ($2$-class) classification problems, the label $y_i \in \{0,1\}$ is binary; for multi-class classification problems, the label is one-hot encoded, $y_i\in\mathbb{R}^{C}$, where $C$ is the total number of classes. We further denote $\mathcal{D}^{(c)}$, $c\in \{1,2, \cdots, C\}$, as the subset of $\mathcal{D}$ that is composed of samples from the $c$-th class.  It is easy to see, $\mathcal{D} = \cup_{c=1}^C \mathcal{D}^{(c)}$.

We assume the labels in  $\mathcal{D}$ are randomly corrupted. Specifically, if denote $\hat{y}$ as the ground truth label of $(\rvx_i, y_i)\in\mathcal{D}$, there exists a non-empty set
\begin{equation}
\mathcal{D}_{noise} \triangleq \{(\rvx_i,y_i)\in\mathcal{D} : y_i \ne \hat{y}_i \}.
\end{equation}
Furthermore, the labels $y_i$ in $\mathcal{D}_{noise}$ is uniformly randomly distributed across all the class labels except $\hat{y}_i$. We call  $\mathcal{D}_{noise}$ as the \emph{noisy subset} and its elements as \emph{noisy samples}. We also define the \emph{clean subset} $\mathcal{D}_{clean}$ as the compliment, i.e., $\mathcal{D}_{clean} = \mathcal{D}\backslash \mathcal{D}_{noise}$, and call its elements as {\it clean samples}.
The {\it noise level} $\delta$ is defined as the ratio  $|\mathcal{D}_{noise}|/|\mathcal{D}|$. In this paper, we set $\delta < 0.5$, i.e., the majority of training samples are not corrupted. 
 We also denote $\hat{\mathcal{D}}$ as the ground-truth-labeled dataset: $\hat{\mathcal{D}}\triangleq \{(\rvx_i, \hat{y}_i)\}_{i=1}^n$.

In addition, there is a test dataset $\bar{\mathcal{D}}$ which is i.i.d. drawn from the same data distribution as the training set $\mathcal{D}$, except that the labels of test set $\bar{\mathcal{D}}$ are not corrupted.

\paragraph{Optimization.} 
Given an arbitrary dataset $\mathcal{S}$ and a model $f$ which is parameterized by $\rvw$ and takes an input $\rvx$, we define the loss function as
\begin{equation}
L(\rvw; \mathcal{S}) = \frac{1}{|\mathcal{S}|}\sum_{(\rvx,y)\in\mathcal{S}}l(\rvw; \rvx_i, y_i),
\end{equation}
with $l(\rvw; \rvx_i, y_i)\triangleq l(f(\rvw;\rvx),y)$ is evaluated on a single sample and is a function of the model output $f(\rvw;\rvx)$ and label $y$. We use the logistic loss for binary classification problems, and  the cross entropy loss for multi-class problems. In this paper, we  use neural networks as the model. In case of binary classification problems, the output layer of $f$ has only one neuron, and there exists a {\it sigmoid} function on top such that the output $f(\rvw;\rvx) \in (0,1)$. In case of multi-class classification problems, the neural network $f$ has $C$ output neurons, and there exists a {\it softmax} function to normalize the outputs. Throughout the paper, we assume that the neural network is large enough so that the training data can be exactly fit. This is usually satisfied when the neural network is over-parameterized \cite{liu2022loss}.

The optimization goal is to minimize the  empirical loss function (i.e., loss function on the training dataset $\mathcal{D}$):
\begin{equation}
 L(\rvw)  \triangleq L(\rvw;\mathcal{D}) = 
 \frac{1}{n}\sum_{i=1}^n l(\rvw; \rvx_i, y_i).
\end{equation}
The above loss function is usually optimized by gradient descent (or its stochastic variants)
 which has the following update form:
\begin{align}
  \rvw_{t+1} & = \rvw_{t} - \eta \nabla L(\rvw_t;\mathcal{D}) \label{eq:gd_update_rule} \\
  & = \rvw_{t} - \eta \frac{1}{n}\sum_{(\rvx_i,y_i)\in \mathcal{D}}\nabla l(\rvw_t; \rvx_i, y_i). \nonumber
\end{align}
Here, $\nabla l(\rvw; \rvx_i, y_i)$ is the gradient of loss $l(\rvw; \rvx_i, y_i)$ w.r.t. the neural network parameters $\rvw$. 

\paragraph{Sample-wise gradients.}
We call $\nabla l(\rvw; \rvx_i, y_i)$  {\it sample-wise gradient}, as it is evaluated on a single  sample, and denote it as $\nabla l_i(\rvw)$ for short. As there are $n$  samples in $\mathcal{D}$, at each point $\rvw$ in the parameter space, we have $n$ sample-wise gradients, each of which is of dimension $p$. 

Denote $h(\rvw;\rvx)$ as the pre-activation output neuron(s) before activation, which is in $\mathbb{R}$ for binary classification, and is in $\mathbb{R}^C$ for multi-class classification. 
The sample-wise gradient for a given sample $(\rvx_i,y_i)$ has the following form \cite{bishop2006pattern}:
\begin{equation}\label{eq:sample_gradient_expression}
\nabla l_i(\rvw) = (f(\rvw;\rvx_i)-y_i) \nabla h(\rvw;\rvx_i).
\end{equation}
Note that the above expression is a scalar-vector multiplication for binary classification, and is a vector-matrix multiplication for multi-class classification.

We note that the sample-wise gradient will be one of our major quantities, and play a fundamental role in the analysis through out this paper. 

\section{Sample-wise gradients at initialization for binary classification}
\label{sec:intialization}
In this section, we analyze the configuration of the $n$ sample-wise gradients at randomly initialization of  neural networks for binary classification.

We start with the expression Eq.(\ref{eq:sample_gradient_expression}) of the sample-wise gradient $\nabla l_i(\rvw)$. For binary classification, $\nabla l_i(\rvw)$ is proportional to the model derivative $\nabla h(\rvw;\rvx_i)$, up to a scalar factor $f(\rvw;\rvx_i)-y_i$. In Section \ref{sec:model_derivative}, we focus on the directions of $\nabla h(\rvw;\rvx_i)$, which is label independent. Then, in Section \ref{sec:sample_gradient_init}, we discuss the effects of the label-related factor $f(\rvw;\rvx_i)-y_i$ and label noise   on the direction of sample-wise gradients.

\subsection{Direction of the model derivative $\nabla h$}\label{sec:model_derivative}
Given any two inputs $\rvx,\rvz\in\mathbb{R}^d$, we denote the angle between $\rvx$ and $\rvz$ in the data space by $\theta_d(\rvx,\rvz)$, and denote the angle between the two vectors $\nabla h(\rvw_0;\rvx)$ and $\nabla h(\rvw_0;\rvz)$ by $\theta_h(\rvx,\rvz)$. In this subsection, we are concerned with the connection between $\theta_h(\rvx,\rvz)$ and $\theta_d(\rvx,\rvz)$ at the network initialization.

Consider the following two types of neural networks: a two-layer linear network $h_1$, which is defined as
\begin{equation}\label{eq:2_layer_nn_linear}
   h_1(\rvw; \rvx) = \frac{1}{\sqrt{m}}\rvv^T A \rvx,
\end{equation}
and a two-layer ReLU network $h_2$, which is defined as
\begin{equation}\label{eq:2_layer_nn_relu}
   h_2(\rvw; \rvx) = \frac{1}{\sqrt{m}}\rvv^T \sigma(A \rvx).
\end{equation}
Here, $\sigma(\cdot)=\max(\cdot,0)$ is the element-wise ReLU activation function,  $\rvv\in\mathbb{R}^{m}$ and $A\in\mathbb{R}^{m\times d}$ are the first layer and second layer parameters, respectively. We denoted $\rvw$ as the collection of all the parameters. We use the NTK parameterization \cite{jacot2018neural} for the networks; namely, each parameter is i.i.d. initialized using the normal distribution $\mathcal{N}(0,1)$ and there exist a scaling factor $1/\sqrt{m}$ explicitly on each hidden layer. 

\begin{figure}[t]
        \centering        
        \includegraphics[width=0.8\columnwidth]
        {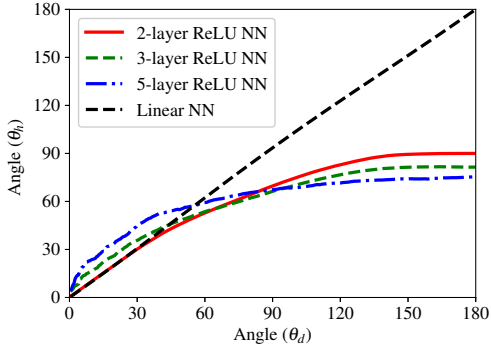}%
        \vspace{-10pt}
        \caption{Relation between $\theta_h$ and $\theta_d$. For both shallow and deep neural networks, similar inputs (small angle $\theta_d$) induce similar model derivatives (small angle $\theta_h$).
        }
        \vspace{-10pt}
    \label{fig:angle_curves}
\end{figure}
The following theorem and its corollary show the relation between $\theta_h(\rvx,\rvz)$ and $\theta_d(\rvx,\rvz)$. (Proofs in Appendix \ref{sec:proof_thm}).
\begin{theorem}\label{thm:angle}
    Consider the two-layer neural networks, $h_1$ and $h_2$, defined in Eq.(\ref{eq:2_layer_nn_linear}) and (\ref{eq:2_layer_nn_relu}), with infinite width $m$. Given any two inputs $\rvx$ and $\rvz$, the angles $\theta_{h_1}(\rvx,\rvz)$ and $\theta_{h_2}(\rvx,\rvz)$ have  the following relations with $\theta_d(\rvx,\rvz)$, at network initialization $\rvw_0$:
    for the linear neural network $h_1$,
    \begin{equation*}
        \theta_{h_1}(\rvx,\rvz)  = \theta_d(\rvx,\rvz);
    \end{equation*}
    for the ReLU neural network $h_2$,
    \begin{equation*}
        \cos\theta_{h_1}(\rvx,\rvz)  = \frac{\pi-\theta_d(\rvx,\rvz)}{\pi}\cos\theta_d(\rvx,\rvz)+\frac{1}{2\pi}\sin \theta_d(\rvx,\rvz).
    \end{equation*}
\end{theorem}

\begin{corollary}\label{corollary:angle}
    Consider the same networks $h_1$ and $h_2$ as in Theorem \ref{thm:angle}. For both networks, the following holds:
    % \vspace{-10pt}
    % \begin{enumerate}
    %     \item given two inputs $\rvx$ and $\rvz$, if $\theta_d(\rvx,\rvz) \ll 1$, then $\theta_{h_i}(\rvx,\rvz) \ll 1$;
    %     \item for any inputs $\rvx$, $\rvz$ and $\rvz'$, if $0 \le \theta_d(\rvx,\rvz) \le \theta_d(\rvx,\rvz') \le \frac{\pi}{2}$, then $0 \le\theta_{h_i}(\rvx,\rvz) \le \theta_{h_i}(\rvx,\rvz')\le \frac{\pi}{2}$, for $i \in \{1,2\}$.
    % \end{enumerate}
    for any inputs $\rvx$, $\rvz$ and $\rvz'$, if $0 \le \theta_d(\rvx,\rvz) \le \theta_d(\rvx,\rvz') \le \frac{\pi}{2}$, then $0 \le\theta_{h_i}(\rvx,\rvz) \le \theta_{h_i}(\rvx,\rvz')\le \frac{\pi}{2}$, for $i \in \{1,2\}$.
    % \vspace{-5pt}
\end{corollary}

The theorem and corollary suggest that: {\it similar inputs (small angle $\theta_d$) induce similar model derivatives (small angle $\theta_h$)}.

\begin{remark}[Not just at initialization]
As discussed in \cite{liu2020linearity,liu2022transition}, the model derivative $\nabla h(\rvw;\rvx)$ is constant  during optimization for infinitely wide neural networks. Hence, the angle $\theta_h$ between model derivatives is also  constant, and Theorem \ref{thm:angle} and Corollary \ref{corollary:angle} apply to any time stamp of the network training.
\end{remark}

\paragraph{Experimental verification.} We experimentally verify the above theoretical results on neural networks with large width. Specifically, we consider six neural networks: three linear networks with $2$, $3$ and $5$ layers, respectively; and  three ReLU networks with $2$, $3$ and $5$ layers. Each hidden layer of each neural network has  $512$ neurons. For each network, we compute the model derivatives $\nabla h$ on the $1$-sphere $\mathcal{S}^1 =\{(\cos \theta_d, \sin \theta_d): \theta_d \in [0, 2\pi)\}$, at the network initialization. Figure \ref{fig:angle_curves} shows the relations between the angle $\theta_h$ and $\theta_d$.\footnote{The curves for the three linear networks are almost identical and not visually distinguishable, we only present the one for $2$-layer linear network in Figure \ref{fig:angle_curves}.}  We observe that the curves for the $2$-layer networks match Theorem \ref{thm:angle}. More importantly, the experiments suggest that the same or similar relations, as well as Corollary \ref{corollary:angle}, still  hold for deep neural networks, although our analysis is conducted on shallow networks. 

Consider the following synthetic dataset (also shown in the left panel of Figure \ref{fig:raw_features_angle}): two separated data clusters in a $2$-dimensional space. We use a $3$-layer ReLU network of width $m=512$ at its initialization to compute the sample-wise model derivatives $\nabla h$.  The right panel of Figure \ref{fig:raw_features_angle} shows the  distributions of angle $\theta_h$ for data pairs from the same cluster (``within'') and from different clusters (``between''). It can be easily seen that the ``within'' distribution has smaller angles $\theta_h$ than the ``between'' distribution, which is expected as the data from the same clusters are more similar.

\begin{figure}[t]
        \centering        
        \hspace*{-0.25cm}
        \includegraphics[width=1.05\columnwidth]
        {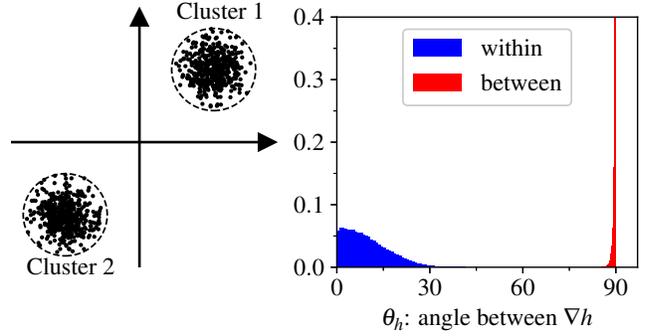}%
        \vspace{-25pt}
        \caption{({\bf Left}) Data visualization: two separated data clusters in $2$-dimensional space. ({\bf Right}) Distributions of angle $\theta_h$ for sample pairs from the same cluster (``within'') and from different clusters (``between'').
        }
        \vspace{-10pt}
    \label{fig:raw_features_angle}
\end{figure}

\subsection{Directions of the sample-wise gradients}\label{sec:sample_gradient_init}
Now, we consider the sample-wise gradients $\nabla l(\rvw_0)$ at initialization, using Eq.(\ref{eq:sample_gradient_expression}). We note that the direction of $\nabla l(\rvw_0)$ is determined by $\nabla h(\rvw_0;\rvx)$ and the label $y$. This is because only the sign (not the magnitude) of $y-f(\rvw_0;\rvx)$ may affect the direction, and the post-activation output $f(\rvw_0;\rvx)$ is always in $(0,1)$ and label $y$ is either $0$ or $1$.

Motivated by this observation, for a fixed class $c\in \{0,1\}$, we consider the following subsets:  clean subset $\mathcal{D}_{clean}^{(c)} \triangleq {\mathcal{D}^{(c)}\cap\mathcal{D}_{clean}}$, 
 noisy subset $\mathcal{D}_{noise}^{(c)}\triangleq{\mathcal{D}^{(c)}\cap\mathcal{D}_{noise}}$, and $\mathcal{D}_{other}^{(c)}\triangleq{\mathcal{D}_{clean}\backslash\mathcal{D}^{(c)}}$. We note that $\mathcal{D}_{noise}^{(c)}$ and $\mathcal{D}_{clean}^{(c)}$ have the same input distribution but different labels $y$, while $\mathcal{D}_{noise}^{(c)}$ and $\mathcal{D}_{other}^{(c)}$ have the same labels $y$ but different input distributions.

For these subsets, we denote their corresponding sets of sample-wise gradients as  $\mathcal{G}_{clean}^{(c)}(\rvw)$, $\mathcal{G}^{(c)}_{noise}(\rvw)$ and $\mathcal{G}^{(c)}_{other}(\rvw)$, respectively. We also define the corresponding {\it subset gradients} as the sum of all sample-wise gradients in the subset: $g_{k}^{(c)}(\rvw) \triangleq \sum_{\nabla l(\rvw)\in\mathcal{G}_{k}^{(c)}}\nabla l(\rvw)$, for $k\in \{clean, ~ noise, ~ other\}$. the direction of $g_{k}^{(c)}(\rvw)$ is the same as that of  the average  gradient in the  subset.

{\bf Direction of sample-wise gradients.} We consider the angles  between sample-wise gradients, and denote by $\theta_g$.
% To further investigate the direction of the sample-wise gradients, we introduce the angle $\theta_g(\rvx_i)$  between $\nabla l_i(\rvw_0)$ and $g_{clean}^{(c)}$ for a given input $\rvx_i$. 

First we consider the clean subset. Within $\mathcal{D}_{clean}^{(c)}$, the factor $f(\rvw;\rvx)-y$ always have the same sign, as the label $y$ is the same. Thus, the directional distribution of $\mathcal{G}_{clean}^{(c)}(\rvw_0)$ is identical to the $\nabla h$ distribution, which has been analyzed in Section \ref{sec:model_derivative}. Presumably, inputs from the same ground truth class tend to be more similar (with small angles $\theta_d$), compared to others. Applying Corollary \ref{corollary:angle}, we expect the angles $\theta_h$, and therefore $\theta_g$ also, within this subset are relatively small.
 % Thus, the angle between $\nabla l_i(\rvw_0)$ is the same as the angle between $\nabla h(\rvw_0;\rvx_i)$ for these inputs. Therefore, Corollary \ref{corollary:angle} still holds for $\theta_g$ replacing $\theta_h$, meaning that if the angle between these inputs and their average is small, then the angle between $\l_i(\rvw_0)$ and $g_{clean}^{(c)}$ is also small.
 The green plots of Figure \ref{fig:histogram_at-initialization} show numerical verification of the $\theta_g$ distributions for the clean subset $\mathcal{D}_{clean}^{(c)}$.\footnote{For illustration purpose, we compare each sample-wise gradient with the average direction, represented by $g^{(c)}_{clean}(\rvw_0)$.} We see that $\theta_g$ tends to concentrate around relatively small angles.

The interesting part is about the noisy subset $\mathcal{D}_{noise}$. This subset shares the same input distribution, hence $\nabla h$ distribution as well, with the clean subset. However, due to different label $y$, the factor $f(\rvw_0;\rvx)-y$ has different sign from that of clean samples, which flips all $\nabla l(\rvw_0)\in\mathcal{G}_{noise}^{(c)}(\rvw_0)$ to the opposite direction of those in $\mathcal{G}_{noise}^{(c)}(\rvw_0)$. As a consequence, sample-wise gradients between $\mathcal{G}_{noise}^{(c)}(\rvw_0)$ and $\mathcal{G}_{clean}^{(c)}(\rvw_0)$
makes large angles $\theta_g$; and the noisy subset gradient $g^{(c)}_{noise}(\rvw_0)$ is sharply opposite to $g^{(c)}_{clean}(\rvw_0)$.
Figure \ref{fig:histogram_at-initialization} experimentally verifies this phenomenon. The red histograms, representing the   $\theta_g$ distribution for noisy subset,   are symmetric to the green ones and locate at large angles. The red  dash lines representing the angle $\theta_g$ between  $g^{(c)}_{noise}(\rvw_0)$ and $g^{(c)}_{clean}(\rvw_0)$ is almost close to $180^{\circ}$. 

Lastly,  the subset $\mathcal{D}_{other}^{(c)}$, having different ground truth labels with the other two subsets, has different input distributions. By Corollary \ref{corollary:angle}, the sample-wise gradients $\nabla l(\rvw_0)$ of this subset are expected to be not align with those of the other two subsets (as shown by the blue histograms in Figure \ref{fig:histogram_at-initialization}). Moreover, the subset gradient $g^{(c)}_{other}(\rvw_0)$ should have a significant  component orthogonal to $g^{(c)}_{noise}(\rvw_0)$ and $g^{(c)}_{clean}(\rvw_0)$ (as shown by the blue dash lines in Figure \ref{fig:histogram_at-initialization}).

{\bf Magnitudes of subset gradients.} We are interested in the magnitudes of $g_{clean}^{(c)}(\rvw_0)$ and $g_{noise}^{(c)}(\rvw_0)$. By definition, for $k \in \{clean, ~ noise\}$, 
\begin{equation*}
g_{k}^{(c)}(\rvw_0) = |D^{(c)}_k| \mathbb{E}[\nabla l] = |D^{(c)}_k| \mathbb{E}[f(\rvw_0;\rvx)-y]\mathbb{E}[\nabla h]
\end{equation*}
where the expectation is taken over the corresponding data subset. We know that $\mathbb{E}[\nabla h]$ is the same for clean and noisy subsets. In addition, $\mathbb{E}[f(\rvw_0;\rvx)-y]$ are opposite for these two subsets, as $\mathbb{E}[f(\rvw_0;\rvx)]$ is $0.5$ by random guess and $y=1$ for one subset and $y=0$ for the other. Hence, we see that the magnitudes $\|g_{k}^{(c)}\|$ are determined by the subset population, and we have 
\begin{equation}\label{eq:init_magnitude_ratio}
\|g_{clean}^{(c)}(\rvw_0)\|/ \|g_{noise}^{(c)}(\rvw_0)\| = (1-\delta)/\delta > 1.
\end{equation}

\begin{figure}[t]
        \centering        
        \hspace{-0.4cm}
        \vspace{-10pt}
        \includegraphics[width=1.04\columnwidth]
        {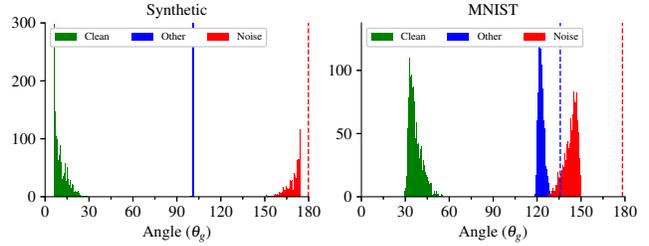}%

        \caption{ The distributions of $\theta_g$. {\bf Left}: synthetic data in Figure \ref{fig:raw_features_angle} ($\delta=0.3$), {\bf Right}: two classes MNIST ((``0'' and 1'', $\delta=0.3$). 
        Dash lines represent subset gradients. 
        Both cases use a 2-layer ReLU neural network.
        }
\vspace{-10pt}
    \label{fig:histogram_at-initialization}
\end{figure}

\section{Learning dynamics of binary classification}
\label{sec:binary_dynamics}

In this section, we analyze the learning dynamics of  binary classification  with label noise  in the training dataset. 

Specifically, we show that in the early stage of training, the dynamics exhibits a clean-priority learning characteristic, due to a dominance of the clean subset in first-order information, i.e., sample-wise gradients. We further show that in later stage of training, this dominance fades away and clean-priority learning terminates, resulting in a fitting of the noisy samples and worsening of the test performance.

We partition the optimization procedure into two stages: early stage which happens before the early stopping point; and later stage which is after the early stopping point.

\begin{figure*}[th]
        \centering   
        \hspace*{-0.3cm}
        \includegraphics[width=0.93\textwidth]
        {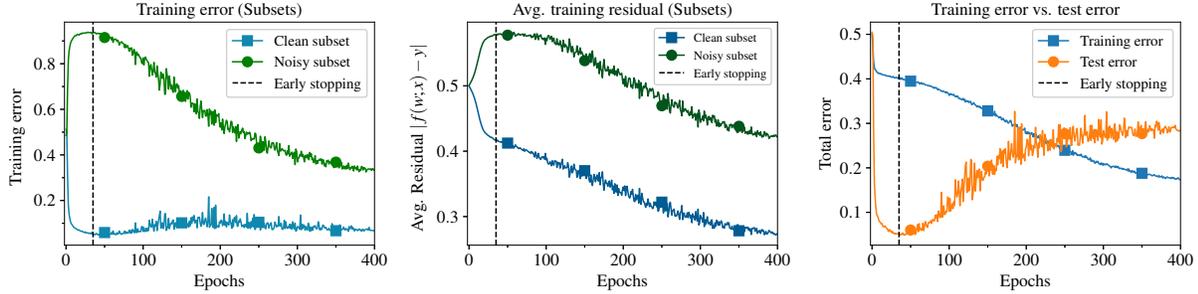}%
        \vspace{-10pt}
        \caption{Learning dynamics on two classes (``7'' and ``9'') of MNIST (noise level $\delta=0.4$) with FCN. {\bf Left}: in the early stage (before the vertical dash line), clean subset error decreases, while  noisy subset error increases. {\bf Middle}: In the early stage, the clean subset average residual $\mathbb{E}_{(\rvx,y)\in \mathcal{D}_{clean}}[|f(\rvw;\rvx)-y|]$ decreases, i.e., on average the network outputs of clean subset move towards the labels, indicating a ``learning'' on the clean subset. One the other hand, the noisy subset average residual, $\mathbb{E}_{(\rvx,y)\in \mathcal{D}_{noise}}[|f(\rvw;\rvx)-y|]$, monotonically increases, indicating that the noisy subset is not-learned. {\bf Right}: total test error and total training error.
        }
\vspace{-5pt}
    \label{fig:mnist2_3plots}
\end{figure*}

\subsection{Initialization \& early stage} 
In Section \ref{sec:intialization}, we have seen that, at initialization,
\begin{equation}
    g_{noise}^{(c)}(\rvw_0) = - \alpha_0 g_{clean}^{(c)}(\rvw_0), 
\end{equation}
with $\alpha_0 \triangleq \delta/(1-\delta) \in (0,1)$. We note that, during training, the model derivative $\nabla h$ for a wide neural network is found to  barely change \cite{liu2020linearity}:
\begin{equation*}
    \nabla h(\rvw_t) = \nabla h(\rvw_0), ~ \forall t>0.
\end{equation*}
By Eq.(\ref{eq:sample_gradient_expression}), this implies that  each sample-wise gradient $\nabla l_i$ keeps its direction unchanged during training (but changes in magnitude through the factor $f(\rvw;\rvx)-y$). Therefore, we make the following assumption:
\begin{assumption}\label{assum:direction_segment}
There exist a time $T>0$ and a sequence $\{\alpha_t\}_{t=0}^T$, with each $\alpha_t \in (0,1)$, such that, for all $t \in [0,T]$ and $c\in \{0,1\}$, the following holds $g_{noise}^{(c)}(\rvw_t) = - \alpha_t g_{clean}^{(c)}(\rvw_t)$.
\end{assumption}

Define $\hat{g}^{(c)}(\rvw)$ as the summation of the sample-wise gradients with ground truth labels, i.e., $\hat{g}^{(c)}(\rvw)=\sum_{(\rvx,\hat{y})\in\hat{\mathcal{D}}^{(c)}} \nabla l(\rvw; \rvx, \hat{y})$.
By the assumption, we have for all $0\le t \le T$ and $c\in\{0,1\}$,
\begin{align}\label{eq:g_corr_epsilon}
 % & g_{noise}^{(c)}(\rvw_t) \approx -\delta \hat{g}^{(c)}(\rvw_t), \\
 & g_{clean}^{(c)}(\rvw_t) = \frac{1}{\alpha_t+1}\hat{g}^{(c)}(\rvw_t).
\end{align}

On the other hand, by definition, we have for  full gradient
\begin{equation}\label{eq:full_grad_g_corr}
 \nabla L(\rvw_t; \mathcal{D}) = \sum_{c} \left(g_{clean}^{(c)}(\rvw_t)+ g_{noise}^{(c)}(\rvw_t)\right).
\end{equation}
Combining Assumption \ref{assum:direction_segment} and Eqs.(\ref{eq:g_corr_epsilon}) and (\ref{eq:full_grad_g_corr}), we get that the full gradient on the training data $\mathcal{D}$ has the same direction with that on the ground-truth-labeled data $\hat{\mathcal{D}}$. Hence, we have the following proposition.
\begin{proposition}[Update rules] \label{thm:update_rule}
Suppose Assumption \ref{assum:direction_segment} holds with time $T>0$ and sequence $\{\alpha_t\}_{t=0}^T \in (0,1)^T$. Then, the gradient descent (with learning rate $\eta$), Eq.(\ref{eq:gd_update_rule}), has the following equivalent update rule
\begin{equation}\label{eq:thm_update_rule}
\rvw_{t+1} = \rvw_t - \eta'_t \nabla L(\rvw_t; \hat{\mathcal{D}}), ~~ \mathrm{for} ~ t\le T,
\end{equation}
with $\eta'_t = \frac{1-\alpha_t}{1+\alpha_t}\eta > 0$ and $\nabla L(\rvw_t; \hat{\mathcal{D}})$ being the gradient evaluated on the ground-truth-labeled dataset $\hat{\mathcal{D}}$.
\end{proposition}
% As we see in Eq.(\ref{eq:g_corr_epsilon}), $g_{clean}^{(c)}(\rvw_0)$ and $g_{noise}^{(c)}(\rvw_0)$ are opposite in directions, and $g_{clean}^{(c)}(\rvw_0)$ dominates in magnitude (since $\delta<0.5$), we have
% \begin{equation}\label{eq:gradient_clean_relation}
%  \nabla L(\rvw_0) \approx (1-2\delta)\sum_c \hat{g}^{(c)}(\rvw_0).
% \end{equation}
% Namely, the {\it direction} of  $\nabla L(\rvw_0)$ is the same as if there is no label noise; the effect of the label noise is only in the {\it magnitude} of $\nabla L(\rvw_0)$, through the positive factor $1-2\delta$.
\begin{remark}[mini-batch scenario] In  mini-batch SGD, similar relation of Eq.(\ref{eq:thm_update_rule}) also holds for a mini-batch estimation $\nabla L(\rvw_0;\mathcal{B})$, as long as  the sampling of the mini-batch is independent of the label noise and the batch size $|\mathcal{B}|$ is not too small such that the majority of samples are clean in the batches. Hence, in the following, we do not explicitly write out the dependence on the mini-batches.
\end{remark}
The theorem states that, after adding label noise to the training dataset, the gradient descent update is equivalent to  the one without label noise (except a different learning rate $\eta'_t < \eta$). In another word, the gradient descent does not essentially ``see'' the noisy data and its update direction is determined only by the clean samples.

\paragraph{Clean-priority learning.}
This theorem implies the following learning characteristics of what we call {\it clean-priority learning}, as we described below. 

{\it Training loss and accuracy on subsets.} The loss $L(\rvw;\mathcal{D}_{clean})$ on the clean subset keeps decreasing, while the loss $L(\rvw;\mathcal{D}_{noise})$ on the noisy subset is increasing, as formally stated in the following Theorem (see the proof in Appendix~\ref{secapp:pf_segment_loss}):
\begin{theorem}\label{coro:segment_loss}
    Suppose Assumption \ref{assum:direction_segment} holds with time $T>0$ and sequence $\{\alpha_t\}_{t=0}^T$, $\alpha_t \in(0,1)$. We have, for all $t\in [0,T]$ and sufficiently small $\eta$,
    \begin{align*}
        &L(\rvw_{t+1}; \mathcal{D}_{clean}) < L(\rvw_t; \mathcal{D}_{clean}); \\
        & L(\rvw_{t+1}; \mathcal{D}_{noise}) > L(\rvw_t; \mathcal{D}_{noise}).
    \end{align*}
\end{theorem}
 Accordingly, the training accuracy on the clean subset is increased, and that on the noisy subset is decreased.

{\it Residual magnitude: $|f(\rvw;\rvx)-y|$.}  As a consequence of the decreasing clean subset loss $L(\rvw;\mathcal{D}_{clean})$, the clean training samples are learned,  in the sense that the network output $f(\rvw;\rvx)$ moves towards its corresponding label $y$, i.e., $|f(\rvw;\rvx)-y|$ decreases on the clean subset. On the other hand, the increase of the  $L(\rvw;\mathcal{D}_{noise})$ results in that, on the noisy subset, the network output $f(\rvw;\rvx)$ moves away from its corresponding label $y$, but towards its ground truth label $\hat{y}$. Namely, the noisy subset is not learnt.

{\it Test loss.} As the test dataset $\bar{\mathcal{D}}$ is not label-corrupted and is drawn from the same data distribution as $\hat{\mathcal{D}}$, it is expected that the update rule in Eq.(\ref{eq:thm_update_rule}) decreases the test loss $L(\rvw; \bar{\mathcal{D}})$.

Figure \ref{fig:mnist2_3plots} shows the clean-priority learning phenomenon on a binary classification of two classes of MNIST. The relevant part is the early stage, i.e., before the early stopping point (left of the vertical dash line). As one can see, in this stage, the  prediction error and noisy subset loss $L(\rvw;\mathcal{D}_{noise})$   keep increasing (See Appendix \ref{sec:app_exp_setup} for subset loss curves). Especially, the prediction error increases from a random guess (error $=0.5$) at initialization towards $100\%$. Meanwhile, the clean subset loss and prediction error  keep decreasing.  Moreover, the average residual magnitude $|f(\rvw;\rvx)-y|$ decreases on the clean subset, but increases on the noisy subset, implying that only clean subset is learnt. These behaviors illustrate that in the early stage the learning dynamics prioritize the clean samples.

In short, in the early stage, the {\it clean-priority learning} prioritizes the learning on clean training samples. The interesting point is that, although it seems impossible to distinguish the clean  from the noisy directly from the data, this prioritization is possible because the model have access to the first-order information, i.e., sample-wise gradients.  Importantly, it is this awareness of the clean samples and this prioritization in the early stage that allow the possibility of achieving test performances better than the noisy level.

\subsection{Early stopping point \& later stage}
As we have seen in the above subsection, the dominance of the magnitude $\|g_{clean}^{(c)}\|$  over $\|g_{noise}^{(c)}\|$ is one of the key reasons to maintain the clean-priority learning in the early stage. However, we shall see shortly that this dominance diminish as the training goes on,  resulting in a final termination of the clean-priority learning. 

\paragraph{Diminishing dominance of the clean gradient.}
 Recall that the sample-wise gradient is proportional to the magnitude of the residual:
$
  \nabla l(\rvw) \propto y-f(\rvw;\rvx).
$
The learning of a sample, i.e., decreased $|y-f(\rvw;\rvx)|$, results in a decrease in the magnitude $|\nabla l(\rvw)|$. As an effect of the clean-priority learning, the residuals magnitude $|f(\rvw;\rvx)-y|$ evolves differently for different data subsets: {\it decreases} on the clean subset $\mathcal{D}_{clean}$, but {\it increases}  on the noisy subset $\mathcal{D}_{noise}$.
This difference leads to the diminishment of the dominance of clean subset $\|g_{clean}^{(c)}(\rvw)\|$, which originates from the dominance of the population of clean training samples.

\begin{theorem}[Diminishing dominance  of the clean gradient]\label{thm:diminishing_dominance}
     Assume the neural network is infinitely wide and the learning rate $\eta$ of the gradient descent is sufficiently small. Suppose Assumption \ref{assum:direction_segment} holds with time $T>0$ and sequence $\{\alpha_t\}_{t=0}^T \in (0,1)^T$.    The sequence $\{\alpha_t\}_{t=0}^T$  monotonically increases: for all $t\in [0,T]$, $\alpha_{t+1} > \alpha_t$. 
\end{theorem}
Please find the proof in Appendix \ref{secapp:pf_diminishing_dominance}.
As $\alpha_t$ measures this clean dominance ($\alpha_t$ close to $1$ means less dominant), this theorem indicates that the dominance diminishes as the training goes on.

Figure \ref{fig:mnist2_ratio} illustrates this diminishing dominance on the two class MNIST classification problem. In the early stage, the ratio $\|g_{clean}^{(c)}(\rvw)\|/\|g_{noise}^{(c)}(\rvw)\|$ starts with a value around the ratio of population $(1-\delta)/\delta = 1.5$, and monotonically decrease to around $1$ at or before the early stopping point, indicating that the dominance vanishes.

\begin{figure}[t]
        \centering   
        % \hspace*{-0.3cm}
        \vspace{-10pt}
        \includegraphics[width=0.85\columnwidth]
        {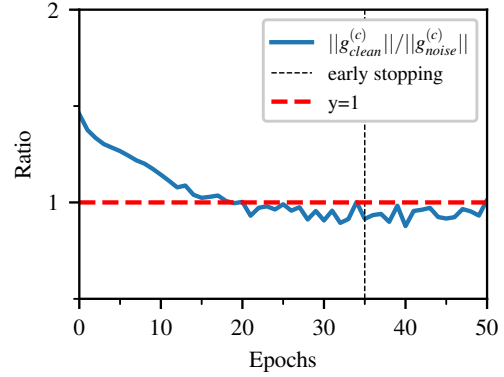}%
        \vspace{-15pt}
        \caption{Diminishing dominance of clean gradient. The ratio $\|g_{clean}^{(c)}(\rvw)\|/\|g_{noise}^{(c)}(\rvw)\|$ monotonically decreases in the early stage (before the vertical dash line), as a consequence of the clean-priority learning, as predicted by Theorem \ref{thm:diminishing_dominance}. Experiment setting is the same as in Figure \ref{fig:mnist2_3plots}.
        }
\vspace{-10pt}
    \label{fig:mnist2_ratio}
\end{figure}

\begin{figure*}[th]
        \centering        
        \includegraphics[width=1.8\columnwidth]
        {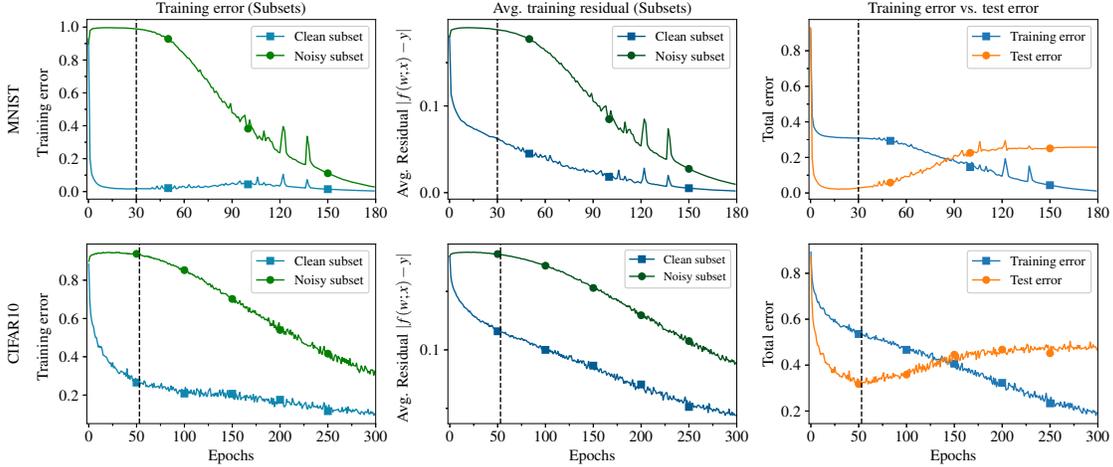}%
        \vspace{-10pt}
        \caption{Learning dynamics on multi-class classification. {\bf Left}: in the early stage (before the vertical dash line), clean subset error decreases, while  noisy subset error increases. {\bf Middle}: In the early stage, the clean subset average residual $\mathbb{E}_{(\rvx,y)\in \mathcal{D}_{clean}}[\|f(\rvw;\rvx)-y\|]$ decreases, i.e., on average the network outputs of clean subset move towards the labels, indicating a ``learning'' on the clean subset. One the other hand, the noisy subset average residual, $\mathbb{E}_{(\rvx,y)\in \mathcal{D}_{noise}}[\|f(\rvw;\rvx)-y\|]$, monotonically increases, indicating that the noisy subset is not-learned. {\bf Right}: total test error and total training error. See subset loss curves in Appendix \ref{sec:app_verification}.
        }
        \vspace{-10pt}
    \label{fig:multi-class}
\end{figure*}

\paragraph{Learning the noisy samples.}
In the later stage (i.e., after the early stopping point), the magnitudes of $\|g_{clean}^{(c)}(\rvw)\|$ and $\|g_{noise}^{(c)}(\rvw)\|$ are similar, and there is no apparent dominance of one over the other. Then, the model and algorithm do not distinguish the clean subset from the noisy one, and there will be no clean-priority learning. In this stage, the model learns both the clean and noisy subsets, aiming at achieving exact fitting of the training data. Ultimately, training errors of both subsets converge to zero. 

It is expected that in this stage the loss and prediction error on the test dataset $\bar{\mathcal{D}}$ become worse, as the learning on the noisy subset contaminates the performance achieved by the clean-priority learning in the earlier stage. 

As illustrated in Figure \ref{fig:mnist2_3plots}, after the early stopping point, the noisy subset starts to be learnt. Specifically, both training loss and error on this subset turn to decrease towards zero; the average residual magnitude $|f(\rvw;\rvx)-y|$ turn to decrease, indicating that the network output $f(\rvw;\rvx)$ is learnt to move towards its (corrupted) label. It is worth to note that the learning on the clean subset is still ongoing, as both training loss and error on this subset keeps decreasing.

In high level, before the first stage, the learning procedure prioritizes the clean training samples, allowing the superior-noise-level performance on the test dataset; in later stage, the learning procedure picks up the noisy samples, worsening the test performance toward the noise-level.

\section{Multi-class classification}
\label{sec:multi_class}

In this section we show that  multi-class classification problems exhibit the same learning dynamics, especially the clean-priority learning, as described in Section \ref{sec:binary_dynamics}. 

For multi-class classification, we consider a variant of the sample-wise gradient, {\it single-logit sample-wise gradient}.

\paragraph{Single-logit sample-wise gradients.}
In a $C$-class classification problem, the neural network $f$ has $C$ output logits, and  the labels are a $C$-dimensional one-hot encoded  vectors. One can view the neural network as $C$ co-existing binary classifiers. Specifically, for each $c\in \{1,2,\cdots, C\}$, the $c$-th logit $f_c$ is a binary classifier, and the $c$-th component of the label $y_c\in \{0,1\}$ is the binary  label for $f_c$.

By Eq.(\ref{eq:sample_gradient_expression}), the sample-wise gradient can be written as
$
    \nabla l(\rvw) =   \sum_{c=1}^C \nabla l_c(\rvw),
$
where 
\begin{equation}\label{eq:single_logit_g}
\nabla l_c(\rvw) \triangleq (f_c(\rvw;\rvx)-y_c) \nabla h_c(\rvw;\rvx)
\end{equation}
is  the {\it single-logit sample-wise gradient},
which only depends on quantities of the corresponding single logit.

We point out that, the cleanness of a sample is only well defined with respect to each single logit, but not to the whole output.
For example, consider a sample with ground truth label $0$ but is incorrectly labeled as class $1$. For all the rest  binary classifiers, except the $0$-th and $1$-st, this sample is always considered as the negative class, as $y_c=0$ for all $c \ne 0$ or  $1$; hence, the noisy sample is considered ``clean'', for these  $C-2$ binary classifiers. Therefore, a noisy sample is not necessarily noisy for all the $C$ binary classifiers. 

With this observation, we  consider the single-logit sample-wise gradient $\nabla l_c(\rvw)$ instead. 

\paragraph{At initialization.} Given $c\in \{1,2,\cdots,C\}$, the $c$-logit sub-network $h_c$ (before  {\it softmax})  is the same as the network $h$  discussed in Section \ref{sec:intialization}, and the output $f_c\in (0,1)$. 
% Note that the {\it softmax} function, although different from {\it sigmoid}, still maps the output to $f_c\in (0,1)$. 
Hence, all the directional analysis for binary case (Section \ref{sec:intialization}) still applies to the single-logit sample-wise gradient $\nabla l_c(\rvw)$. See Appendix \ref{sec:app_verification} for numerical verification.

Different from the {\it sigmoid} output activation which tends to predict an average of $0.5$ before training, the {\it softmax} has an average output $f_c$ around $1/C$ with random guess at initialization. This leads to $\mathbb{E} |f_c(\rvw_0;\rvx)-y_c| = 1-1/C$ when $y_c=1$, and $\mathbb{E} |f_c(\rvw_0;\rvx)-y_c| = 1/C$ when $y_c=0$. Recalling that $\mathcal{D}^{(c)}_{clean}$ and $\mathcal{D}^{(c)}_{noise}$ (hence the corresponding $\nabla h$) have the same distribution, using Eq.(\ref{eq:single_logit_g}) we have 
\begin{subequations}\label{eq:g_corr_epsilon_multi}
\begin{align}
 & g_{noise}^{(c)}(\rvw_0) \approx -\delta \hat{g}^{(c)}(\rvw_0)/(C-1), \\
 & g_{clean}^{(c)}(\rvw_0) \approx (1-\delta)\hat{g}^{(c)}(\rvw_0).
\end{align}
\end{subequations}
Therefore, we have the dominance of $\|g_{clean}^{(c)}\|$ over $\|g_{noise}^{(c)}\|$ at initialization, with a ratio  
\[\|g_{clean}^{(c)}(\rvw_0)\|/\|g_{clean}^{(c)}(\rvw_0)\| \approx (C-1)(1-\delta)/\delta.\]

\begin{figure}[t]
        \centering        
        \vspace{-5pt}
        \includegraphics[width=1.00\columnwidth]
        {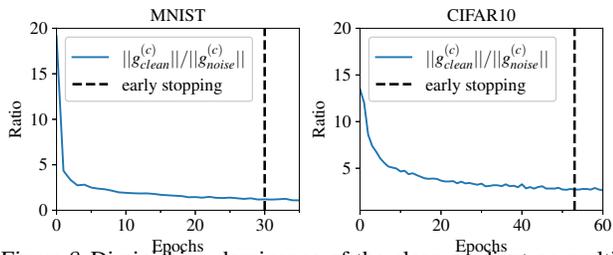}%
        \vspace{-20pt}
        \caption{Diminishing dominance  of the clean gradient on multi-class classification. The ratio $\|g_{clean}^{(c)}(\rvw)\|/\|g_{noise}^{(c)}(\rvw)\|$ monotonically decreases in the early stage (before the vertical dash line), as a consequence of the clean-priority learning. {\bf Left}: MNIST on CNN (noise level $\delta=0.3$); {\bf Right}: CIFAR-10 on ResNet (noise level $\delta=0.4$).
        }
        \vspace{-8pt}
    \label{fig:mnist10_ratio}
\end{figure}

\paragraph{Learning dynamics.} As the configuration of $\nabla l_c$ is similar to that of a binary classification, we expect similar learning dynamics as discussed in Section \ref{sec:binary_dynamics}, especially the clean-priority learning,  happen for multi-class classification. 

\begin{figure}[ht]
        \centering        
        \includegraphics[width=1.0\columnwidth]
        {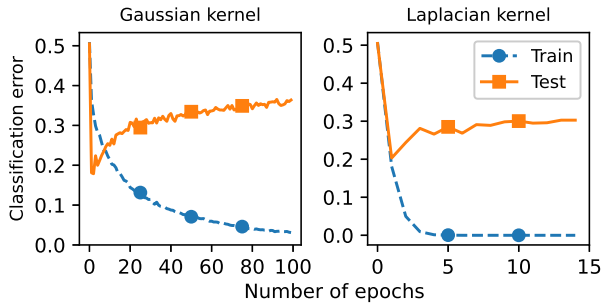}%
        \vspace{-10pt}
        \caption{Classification errors on training and test datasets of MNIST, 7 versus 9 subset, using kernel machine with both Laplacian and Gaussian kernels. We added label noise level of 0.4 to the training set. On test data set, the classification error exhibits a U-shaped curve, and can be significantly lower than the noise level during training.
        }
        \vspace{-10pt}
    \label{fig:kernel}
\end{figure}

We conduct  experiments to classify the MNIST (with added label noise $\delta=0.3$) and CIFAR-10 (with added label noise $\delta=0.4$) datasets using a CNN and a ResNet, respectively. As is shown in Figure \ref{fig:mnist10_ratio} and Figure \ref{fig:multi-class}, in most of the early stage, the clean subset has clean dominance over the noise subset and the dynamics shows the clean-priority learning characteristic, decreasing the clean subset error and residual, but increasing the noisy subset error and residual. Furthermore, the dominance of the clean subset monotonically decreases (Figure \ref{fig:mnist10_ratio}) until the early stopping point. In the later stage, the networks start to learn the noisy subsets. See the detailed experimental setup in Appendix \ref{sec:app_exp_setup}.

\section{Discussion and Future work}
In this section, we aim to provide some insights into the relationship between the clean-priority phenomena described in this paper and previous works, as well as how this phenomena may manifest in other gradient descent-based learning algorithms. 

Firstly, previous studies observed that,  for certain very large models,  the test classification error may exhibit a second descent in the later stages  of training \cite{nakkiran2021deep}. In such scenarios, our analysis still holds for the first descent and the subsequent ascent that follows it. Regarding the second descent, we hypothesize that it may be connected to certain types of underlying feature learning dynamics, for example, the Expected Gradient Outer Product (EGOP) \cite{radhakrishnan2022feature}. However, further investigation is needed to confirm this hypothesis, and we leave it as a topic for future research.

Secondly, as can be seen in Figure \ref{fig:intro_illu}, \ref{fig:mnist2_3plots}, and \ref{fig:multi-class},  This occurs sometime after the early stopping point. It suggests that while the model is fitting the noise, the test performance at convergence is not catastrophic, and still outperforms the random guess. This, in turn, implies that neural networks demonstrate a tempered over-fitting behavior, as described in \cite{mallinarbenign}.

% Finally, it is worth mentioning that the analysis presented in this paper can potentially be extended to other models using SGD. For instance, in Figure \ref{fig:kernel}, we solved a kernel machine for a subset of the MNIST dataset, with added label noise, to classify the digit 7 versus 9 using EigenPro2.0 \cite{MLSYS2019_a4a042cf}, a gradient-based method for solving kernel machines. We observed a  U-shaped curve for test performance, which is similar to that we examined in this paper for neural networks. This suggests that a similar analysis could be conducted for other models using SGD.

\subsubsection*{Acknowledgments}
We are grateful for the support from the National Science Foundation (NSF) and the Simons Foundation for the Collaboration on the Theoretical Foundations of Deep Learning (\url{https://deepfoundations.ai/}) through awards DMS-2031883 and \#814639  and the TILOS institute (NSF CCF-2112665).
This work used NVIDIA V100 GPUs NVLINK and HDR IB (Expanse GPU) at SDSC Dell Cluster through allocation TG-CIS220009 and also, Delta system at the National Center for Supercomputing Applications through allocation bbjr-delta-gpu from the Advanced Cyberinfrastructure Coordination Ecosystem: Services \& Support (ACCESS) program, which is supported by National Science Foundation grants \#2138259, \#2138286, \#2138307, \#2137603, and \#2138296.


\bibliography{example_paper}
\bibliographystyle{icml2023}

%%%%%%%%%%%%%%%%%%%%%%%%%%%%%%%%%%%%%%%%%%%%%%%%%%%%%%%%%%%%%%%%%%%%%%%%%%%%%%%
%%%%%%%%%%%%%%%%%%%%%%%%%%%%%%%%%%%%%%%%%%%%%%%%%%%%%%%%%%%%%%%%%%%%%%%%%%%%%%%
% APPENDIX
%%%%%%%%%%%%%%%%%%%%%%%%%%%%%%%%%%%%%%%%%%%%%%%%%%%%%%%%%%%%%%%%%%%%%%%%%%%%%%%
%%%%%%%%%%%%%%%%%%%%%%%%%%%%%%%%%%%%%%%%%%%%%%%%%%%%%%%%%%%%%%%%%%%%%%%%%%%%%%%
\newpage
\appendix
\onecolumn
\section{Technical proofs}

\subsection{Proof of Theorem \ref{thm:angle}}\label{sec:proof_thm}
We restate Theorem \ref{thm:angle} below for easier reference.
\begin{theorem}[Theorem \ref{thm:angle}]
    Consider the two-layer neural networks, $h_1$ and $h_2$, defined in Eq.(\ref{eq:2_layer_nn_linear}) and (\ref{eq:2_layer_nn_relu}), with infinite width $m$. Given any two inputs $\rvx$ and $\rvz$, the angles $\theta_h(\rvx,\rvz)$ and $\theta_d(\rvx,\rvz)$ satisfies the following, at network initialization $\rvw_0$:
    for the linear neural network $h_1$,
    \begin{equation}
        \theta_h(\rvx,\rvz)  = \theta_d(\rvx,\rvz);
    \end{equation}
    for the ReLU neural network $h_2$,
    \begin{equation}
        \cos\theta_h(\rvx,\rvz)  = \frac{\pi-\theta_d(\rvx,\rvz)}{\pi}\cos\theta_d(\rvx,\rvz)+\frac{1}{2\pi}\sin \theta_d(\rvx,\rvz).
    \end{equation}
\end{theorem}
\begin{proof}
We  prove for  the two neural networks separately.

{\bf Two-layer linear network $h_1$.} 
According to the definition of $h_1$ (Eq.(\ref{eq:2_layer_nn_linear})), its model derivative $\nabla h_1(\rvw_0; \rvx)$ (at initialization $\rvw_0$ with input $\rvx$) can be written as
\begin{equation}
    \nabla h_1(\rvw_0; \rvx) = \frac{1}{\sqrt{m}}\mathrm{flattern}\left(A_0\rvx, \rvv_0 \rvx^T\right).
\end{equation}
Here, $A_0$ and $\rvv_0$ are the initialization instance of the parameters $A$ and $\rvv$, respectively. 

Hence, for the two inputs  $\rvx$ and  $\rvz$, the inner product 
\begin{equation}
    \langle \nabla h_1(\rvw_0; \rvx), \nabla h_1(\rvw_0; \rvz)\rangle = \frac{1}{m} \rvz^T A_0^TA_0 \rvx + \frac{1}{m}\|\rvv_0\|^2 \rvz^T \rvx.
\end{equation}
When the network width $m$ is infinite, we have the following lemma (see proof in Appendix \ref{sec:app_pf_lemma}).
\begin{lemma}\label{lemma:rand_m}
Consider a matrix $A\in\mathbb{R}^{m\times d}$, with each entry of $A$ is i.i.d. drawn from $\mathcal{N}(0,1)$. In the limit of $m\to\infty$, 
\begin{equation}
  \frac{1}{m}A^TA \to I_{d\times d}, ~~ \textrm{in  probability.}
\end{equation}
\end{lemma}

Using this lemma, we have in the limit of $m\to \infty$,
\begin{equation}
  \langle \nabla h_1(\rvw_0; \rvx), \nabla h_1(\rvw_0; \rvz)\rangle = \rvz^T I_{d\times d}\rvx + I_{1\times 1}\rvz^T \rvx =
  2 \rvz^T \rvx.
\end{equation}
Therefore,
\begin{align*}
    \cos \theta_h (\rvx, \rvz) &= \frac{\langle \nabla h_1(\rvw_0; \rvx), \nabla h_1(\rvw_0; \rvz)\rangle}{\|\nabla h_1(\rvw_0; \rvx)\|\|\nabla h_1(\rvw_0; \rvz)\|} = \frac{2 \rvz^T \rvx}{\sqrt{2}\|\rvz\| \cdot \sqrt{2}\|\rvx\|} = \cos \theta_d (\rvx, \rvz).
\end{align*}

{\bf Two-layer ReLU network $h_2$.} 
According to the definition of $h_2$ (Eq.(\ref{eq:2_layer_nn_relu})), its model derivative $\nabla h_2(\rvw_0; \rvx)$ (at initialization $\rvw_0$ with input $\rvx$) can be written as
\begin{equation}
    \nabla h_2(\rvw_0; \rvx) = \frac{1}{\sqrt{m}}\mathrm{flattern}\left(A_0\rvx \mathbb{I}_{\{A_0\rvx \ge 0\}}, \left(\rvv_0\mathbb{I}_{\{A_0\rvx \ge 0\}}\right) \rvx^T\right),
\end{equation}
where $\mathbb{I}_{\{\cdot\}}$ is the (element-wise) indicator function.

For any two inputs $\rvx$ and $\rvz$, the inner product
\begin{align*}
    \langle \nabla h_2(\rvw_0; \rvx), \nabla h_2(\rvw_0; \rvz)\rangle &= \frac{1}{m} \sum_{i=1}^m \rvz^T \rva_i \rva_i^T \rvx\mathbb{I}_{\{\rva_i^T\rvx \ge 0, \rva_i^T \rvz \ge 0\}} + \frac{1}{m}\sum_{i=1}^m v_i^2 \mathbb{I}_{\{\rva_i^T\rvx \ge 0, \rva_i^T \rvz \ge 0\}}\rvz^T\rvx
\end{align*}
where $\rva_i^T$ is the $i$-th row of the matrix $A_0$, and $v_i$ is the $i$-th component of the vector $\rvv_0$. In the limit of infinite width $m\to \infty$, this inner product converges to 
\begin{align*}
    \langle \nabla h_2(\rvw_0; \rvx), \nabla h_2(\rvw_0; \rvz)\rangle &= \underbrace{\mathbb{E}_{\rva \sim \mathcal{N}(0,I_{d\times d})} \left[\rvz^T \rva \rva^T \rvx\mathbb{I}_{\{\rva^T\rvx \ge 0, \rva^T \rvz \ge 0\}}\right]}_{\mathcal{A}} + \underbrace{\mathbb{E}_{\rva \sim \mathcal{N}(0,I_{d\times d}), v\sim \mathcal{N}(0,1)} \left[ v^2 \mathbb{I}_{\{\rva^T\rvx \ge 0, \rva^T \rvz \ge 0\}}\rvz^T\rvx\right]}_{\mathcal{B}}.
\end{align*}

As the vector $\rva$  is isotropically distributed and only appears in inner products with $\rvx$ and $\rvz$, we can assume without loss of generality that (without ambiguity, we write $\theta_d(\rvx,\rvz)$ as $\theta_d$, and $\theta_h(\rvx,\rvz)$ as $\theta_h$, for short):
\begin{equation}\label{eq:x_z_assump}
   \rvx = (\|\rvx\|, 0, 0, \cdots, 0), ~~ \rvz = (\|\rvz\|\cos \theta_d, \|\rvz\|\sin \theta_d, 0, 0, \cdots, 0).
\end{equation}
In this setting, the only relevant parts of $\rva$ are  its first two components $a_1$ and $a_2$. We write $\rvb = (a_1, a_2, 0, 0, \cdots, 0)$.

With Eq.(\ref{eq:x_z_assump}), for the term $\mathcal{A}$, we have
\begin{align*}
\mathcal{A} &= \mathbb{E}_{\rvb \sim \mathcal{N}(0,I_{2\times 2})} \left[\rvz^T \rvb \rvb^T \rvx\mathbb{I}_{\{\rvb^T\rvx \ge 0, \rvb^T \rvz \ge 0\}}\right]\\
&= \mathbb{E}_{\rvb \sim \mathcal{N}(0,I_{2\times 2})}[\|\rvb\|^2] \cdot \|\rvx\|\|\rvz\|\cdot \frac{1}{2\pi}\int_{\theta_d-\frac{\pi}{2}}^{\frac{\pi}{2}} \cos(\phi)\cos(\theta_d-\phi) \, d\phi\\
&\overset{(a)}{=} 2 \cdot \|\rvx\|\|\rvz\|\cdot \frac{1}{2\pi}\int_{\theta_d-\frac{\pi}{2}}^{\frac{\pi}{2}}\frac{1}{2}\left(\cos \theta_d + \cos(\theta_d-2\phi)\right) \, d\phi \\
&= \frac{1}{2\pi}\|\rvx\|\|\rvz\| \left( (\pi-\theta_d)\cos \theta_d + \sin \theta_d   \right)\\
&= \frac{\pi-\theta_d}{2\pi} \rvz^T \rvx + \frac{1}{2\pi}\|\rvx\|\|\rvz\|\sin\theta_d.
\end{align*}
In the equality $(a)$ above, we applied trigonometric subtraction formula for $\cos(\theta_d-\phi)$ and $\cos(\theta_d-2\phi)$, as well as the double angle formulas.

For the term $\mathcal{B}$, we have
\begin{align*}
\mathcal{B} &= \mathbb{E}_{v\sim \mathcal{N}(0,1)}\left[v^2\right]\cdot  \frac{1}{2\pi}\int_{\theta_d-\frac{\pi}{2}}^{\frac{\pi}{2}}\rvz^T\rvx\, d\phi \\
&= \frac{\pi-\theta_d}{2\pi} \rvz^T \rvx.
\end{align*}
Combining $\mathcal{A}$ and $\mathcal{B}$, we have
\begin{equation}
\langle \nabla h_2(\rvw_0; \rvx), \nabla h_2(\rvw_0; \rvz)\rangle = \frac{\pi-\theta_d}{\pi} \rvz^T \rvx + \frac{1}{2\pi}\|\rvx\|\|\rvz\|\sin\theta_d.
\end{equation}
Therefore, 
\begin{equation}
 \cos \theta_h = \frac{\langle \nabla h_1(\rvw_0; \rvx), \nabla h_1(\rvw_0; \rvz)\rangle}{\|\nabla h_1(\rvw_0; \rvx)\|\|\nabla h_1(\rvw_0; \rvz)\|} = \frac{\pi-\theta_d}{\pi}\cos \theta_d + \frac{1}{2\pi}\sin \theta_d.
\end{equation}
Therefore, we conclude the proof of the theorem.
\end{proof}
\subsection{Proof of Corollary \ref{corollary:angle}}
We restate Corollary \ref{corollary:angle} below.

\begin{corollary}[Corollary \ref{corollary:angle}]
    Consider the same networks $h_1$ and $h_2$ as in Theorem \ref{thm:angle}. For both networks, the following holds:
    \begin{enumerate}
        \item given two inputs $\rvx$ and $\rvz$, if $\theta_d(\rvx,\rvz) \ll 1$, then $\theta_h(\rvx,\rvz) \ll 1$;
        \item for any three inputs $\rvx$, $\rvz$ and $\rvz'$, if $0 \le \theta_d(\rvx,\rvz) \le \theta_d(\rvx,\rvz') \le \frac{\pi}{2}$, then $0 \le\theta_h(\rvx,\rvz) \le \theta_h(\rvx,\rvz')\le \frac{\pi}{2}$.
    \end{enumerate}
\end{corollary}
\begin{proof}
For the infinitely wide two-layer linear network $h_1$, we have seen in Theorem \ref{thm:angle} that 
\begin{equation}
\theta_h(\rvx,\rvz) = \theta_d(\rvx,\rvz), ~~ \theta_h(\rvx,\rvz') = \theta_d(\rvx,\rvz').
\end{equation}
Hence, the corollary is trivial for $h_1$.

Below, we consider the infinitely wide two-layer ReLU network $h_2$.
By Theorem \ref{thm:angle}, we have the connection
\begin{equation}\label{eq:relation_in_pf}
        \cos\theta_h(\rvx,\rvz)  = \frac{\pi-\theta_d(\rvx,\rvz)}{\pi}\cos\theta_d(\rvx,\rvz)+\frac{1}{2\pi}\sin \theta_d(\rvx,\rvz).
\end{equation}
If $\theta_d(\rvx,\rvz) \ll 1$, then using Taylor expansion on the right hand side, we have 
\begin{equation*}
    \cos\theta_h(\rvx,\rvz) = 1- \frac{1}{2\pi}\theta_d(\rvx,\rvz) -\frac{1}{2}\theta_d^2(\rvx,\rvz) + o\left(\theta_d^2(\rvx,\rvz)\right).
\end{equation*}
Hence, $1-\cos\theta_h(\rvx,\rvz) \ll 1$, which implies $\theta_h(\rvx,\rvz)\ll 1$. We conclude the first statement.

For the second statement of the corollary, note that when $\theta_d$ is in $[0,\pi/2]$ the right hand side (R.H.S.) of Eq.(\ref{eq:relation_in_pf}) is always positive, resulting in $\theta_h\in [0,\pi/2]$.
For the rest of the statement, it suffices to prove the monotonicity of the relation in $[0,\pi/2]$. This is done by having the monotonicity of the functions of $\arccos(\cdot)$ and R.H.S.. To see the latter, we write
\begin{equation}
\frac{d(\mathrm{R.H.S.})}{d\theta_d} = -\frac{1}{2\pi}\cos\theta_d - \left(1-\frac{\theta_d}{\pi}\right)\sin \theta_d,
\end{equation}
which is always non-positive in $[0,\pi/2]$. Hence, we are done with the second statement.
\end{proof}

\subsection{Proof of Lemma \ref{lemma:rand_m}}
\label{sec:app_pf_lemma}
\begin{proof}
We denote $A_{ij}$ as the $(i,j)$-th entry of the matrix $A$. Therefore,  $(A^TA)_{ij}=\sum_{k=1}^m A_{ki}A_{kj}$. 
First we find the mean of each $(A^TA)_{ij}$. Since $A_{ij}$ are i.i.d. and has zero mean, we can easily see that for any index $k$,
\begin{align*}
    \mathbb{E}[A_{ki}A_{kj}] =\begin{cases}
                                1,& \text{if } i=j\\
                                0,              & \text{otherwise}
                                \end{cases}.
\end{align*}
Consequently,
\begin{align*}
    \mathbb{E}[(\frac{1}{m}A^TA)_{ij}] =\begin{cases}
                                1,& \text{if } i=j\\
                                0,              & \text{otherwise}
                                \end{cases}.
\end{align*}
That is $\mathbb{E}[\frac{1}{m}A^TA]=I_d$. 

Now we consider the variance of each $(A^TA)_{ij}$. If $i\neq j$ we can explicitly write,
\begin{align*}
Var\left[\frac{1}{m}(A^TA)_{ij}\right] &= \frac{1}{m^2}\cdot \mathbb{E}\left[\sum_{k_1=1}^m\sum_{k_2=1}^m A_{k_1i}A_{k_1j}A_{k_2i}A_{k_2j}\right]\\
&=\frac{1}{m^2}\cdot \sum_{k_1=1}^m\sum_{k_2=1}^m \mathbb{E}\left[A_{k_1i}A_{k_1j}A_{k_2i}A_{k_2j}\right]\\
&= \frac{1}{m^2}\left(\sum_{k=1}^m \mathbb{E}\left[A_{ki}^2A_{kj}^2\right] + \sum_{k_1 \ne k_2}\mathbb{E}\left[A_{k_1i}A_{k_1j}A_{k_2i}A_{k_2j}\right]\right) \\
&=  \frac{1}{m^2}\left(\sum_{k=1}^m \mathbb{E}\left[A_{ki}^2\right]\mathbb{E}\left[A_{kj}^2\right] + \sum_{k_1 \ne k_2}\mathbb{E}[A_{k_1i}]\mathbb{E}[A_{k_1j}]\mathbb{E}[A_{k_2i}]\mathbb{E}[A_{k_2j}]\right)\\
&=\frac{1}{m^2}\cdot (m + 0) =  \frac{1}{m}.
\end{align*}
% (a) holds because if $k_1\neq k_2$, then  $\mathbb{E}\left[A_{k_1i}A_{k_1j}A_{k_2i}A_{k_2j}\right]=\mathbb{E}(A_{k_1i})\mathbb{E}(A_{k_1j})\mathbb{E}(A_{k_2i})\mathbb{E}(A_{k_2j})=0$. On the other hand if $k_1=k_2=k$, then $\mathbb{E}\left[A_{k_1i}A_{k_1j}A_{k_2i}A_{k_2j}\right] = \mathbb{E}(A_{ki}^2)\mathbb{E}(A_{kj}^2)=1$.

In the case of $i=j$, then,
\begin{align}
Var\left[\frac{1}{m}(A^TA)_{ii}\right] &= \frac{1}{m^2}\cdot Var\left[\sum_{k=1}^m A_{ki}^2\right] = \frac{1}{m^2}\cdot \sum_{k=1}^m Var\left[ A_{ki}^2\right] \overset{(a)}{=} \frac{1}{m^2}(m\cdot2) = \frac{2}{m}.
\end{align}
In the equality (a) above, we used the fact that $A_{ki}^2\sim \chi^2(1)$. Therefore, $\lim_{m \to \infty} Var(\frac{1}{m}(A^TA))=0$.

Now applying Chebyshev's inequality we get,
\begin{align}
Pr(|\frac{1}{m}A^TA-I_d|\geq\epsilon)\leq \frac{Var(\frac{1}{m}(A^TA))}{\epsilon}
\end{align}
Obviously for any $\epsilon\geq0$ as $m\rightarrow\infty$, the R.H.S. goes to zero. Thus, $\frac{1}{m}A^TA \to I_{d\times d}, ~~ \textrm{in  probability.}$
\end{proof}
\subsection{Proof of Theorem \ref{coro:segment_loss}}\label{secapp:pf_segment_loss}

\begin{proof}
First, note that the clean data has the same distribution as the noiseless  (ground-truth-labelled) data. Hence, $L(\rvw;\mathcal{D}_{clean}) = L(\rvw;\hat{\mathcal{D}})$. By Proposition \ref{thm:update_rule}, the gradient descent minimizes $ L(\rvw;\hat{\mathcal{D}})$, as long as the learning rate $\eta$ is small enough to avoid over-shooting. Therefore, it is straightforward to get that the gradient descent also decreases the clean subset loss $L(\rvw;\mathcal{D}_{clean})$.

Let's consider the noisy subset $\mathcal{D}_{noise}$. Combining Assumption \ref{assum:direction_segment} and Eqs.(\ref{eq:g_corr_epsilon}) and (\ref{eq:full_grad_g_corr}), we get
\begin{equation}
    \nabla L(\rvw_t;\hat{\mathcal{D}}) = -\frac{1-\alpha_t}{\alpha_t}\sum_c g^{(c)}_{noise}(\rvw_t) = -\frac{1-\alpha_t}{\alpha_t} \nabla L(\rvw_t,\mathcal{D}_{noise}).
\end{equation}
We note that the factor $-\frac{1-\alpha_t}{\alpha_t}$ is negative, indicating that the gradient descent update, Eq.(\ref{eq:thm_update_rule}), is in opposite direction of minimizing the noisy subset $L(\rvw_t,\mathcal{D}_{noise})$. Hence, we get that 
$L(\rvw_{t+1},\mathcal{D}_{noise}) > L(\rvw_t,\mathcal{D}_{noise})$.
\end{proof}

\subsection{Proof of Theorem \ref{thm:diminishing_dominance}}\label{secapp:pf_diminishing_dominance}
\begin{proof}
    First note that an infinitely wide feedforward neural network (before the activation function on output layer) is linear in its parameters, and can be written as \cite{liu2020linearity, zhu2022transition}:
    \begin{equation}
        h(\rvw;\rvx) = h(\rvw_0;\rvx) + \nabla h(\rvw_0;\rvx)^T (\rvw-\rvw_0),
    \end{equation}
    where $\nabla h(\rvw_0;\rvx)$ is constant during training. As is known, the logistic regression loss (for an arbitrary $\mathcal{S}$) on a linear model
    \begin{equation}
        L(\rvw;\mathcal{S}) = \sum_{(\rvx,y)\in\mathcal{S}} - y \log f(\rvw;\rvx) - (1-y) \log (1-f(\rvw;\rvx)),
    \end{equation}
     is a convex function with respect to the parameters $\rvw$, where $f(\rvw;\rvx) = sigmoid (h(\rvw;\rvx))= 1/(1+\exp(-h(\rvw;\rvx)))$. Hence, at any point $\rvw$ we have the Hessian matrix $H(\rvw;\mathcal{S})$ of the logistic regression loss $L(\rvw;\mathcal{S})$ is positive definite.

Now, consider the point $\rvw_{t+1} = \rvw_t - \eta L(\rvw_t;\mathcal{D})$ with a sufficiently small step size $\eta$. Using Assumption \ref{assum:direction_segment} and Eq.(\ref{eq:full_grad_g_corr}), we can also write $\rvw_{t+1}$ as
\begin{align*}
    & \rvw_{t+1} = \rvw_t - \eta (1-\alpha_t) L(\rvw_t;\mathcal{D}_{clean}), ~ \mathrm{or}\\
    & \rvw_{t+1} = \rvw_t + \eta \frac{1-\alpha_t}{\alpha_t}L(\rvw_t;\mathcal{D}_{noise}).
\end{align*}
For $\mathcal{D}_{clean}$, we have
\begin{equation}
    \nabla L(\rvw_{t+1};\mathcal{D}_{clean}) = \nabla L(\rvw_t;\mathcal{D}_{clean}) + H(\xi;\mathcal{D}_{clean})(\rvw_{t+1}-\rvw_t),
\end{equation}
with $\xi$ being some point between $\rvw_t$ and $\rvw_{t+1}$. Then,
\begin{align*}
    \|\nabla L(\rvw_{t+1};\mathcal{D}_{clean})\|^2 = \|\nabla L(\rvw_t;\mathcal{D}_{clean})\|^2 - 2 \eta (1-\alpha_t) \nabla L(\rvw_t;\mathcal{D}_{clean})^T H(\xi;\mathcal{D}_{clean})\nabla L(\rvw_t;\mathcal{D}_{clean}) + O(\eta^2).
\end{align*}
By the convexity of the loss function (i.e., the positive definiteness of Hessian $H$), we easily get $$\|\nabla L(\rvw_{t+1};\mathcal{D}_{clean})\|^2 < \|\nabla L(\rvw_t;\mathcal{D}_{clean})\|^2.$$

Similarly for $\mathcal{D}_{noise}$, 
\begin{align*}
    \|\nabla L(\rvw_{t+1};\mathcal{D}_{noise})\|^2 = \|\nabla L(\rvw_t;\mathcal{D}_{noise})\|^2 + 2 \eta \frac{1-\alpha_t}{\alpha_t} \nabla L(\rvw_t;\mathcal{D}_{noise})^T H(\xi';\mathcal{D}_{noise})\nabla L(\rvw_t;\mathcal{D}_{noise}) + O(\eta^2).
\end{align*}
Hence, for small $\eta$, we get $$L(\rvw_{t+1};\mathcal{D}_{noise})\|^2 > \|\nabla L(\rvw_t;\mathcal{D}_{noise})\|^2.$$

% By Theorem \ref{coro:segment_loss}, we have $L(\rvw_{t+1}; \mathcal{D}_{clean}) < L(\rvw_t; \mathcal{D}_{clean})$ and $L(\rvw_{t+1}; \mathcal{D}_{noise}) > L(\rvw_t; \mathcal{D}_{noise})$. With a similar analysis as above, we obtain 
% \begin{equation}\label{eq:pf_app_magnitude}
%     \|\nabla L(\rvw_{t+1};\mathcal{D}_{clean})\|^2 < \|\nabla L(\rvw_t;\mathcal{D}_{clean})\|^2, ~~  \|\nabla L(\rvw_{t+1};\mathcal{D}_{noise})\|^2 > \|\nabla L(\rvw_t;\mathcal{D}_{noise})\|^2.
% \end{equation}
Noting that
\[
\nabla L(\rvw_t;\mathcal{D}_{noise}) = \sum_c g_{noise}^{(c)}(\rvw_t) = -\alpha_t \sum_c g_{clean}^{(c)}(\rvw_t) = -\alpha_t \nabla L(\rvw_t;\mathcal{D}_{clean}),
\] 
 we obtain
\begin{equation}
    \alpha_{t+1} > \alpha_t.
\end{equation}
Therefore, we conclude the proof of the theorem.
\end{proof}
The high-level idea of the above proof is that: (locally) decreasing a convex function $L$ along the opposite gradient direction, $-\nabla L$, results in shrinking the magnitude of the gradient; (locally) increasing a convex function $L$ along the gradient direction $\nabla L$ results in magnifying the gradient magnitude. 

\section{Experimental setup details}
\label{sec:app_exp_setup}

\paragraph{Binary classification on two class of MNIST.}  We extract two classes, the images with digits ``7'' and ``9'', out from the MNIST datasets, and injected $30\%$ random label noise into each class in the training dataset (i.e., labels of $30\%$ randomly selected samples are flipped to the other class), leaving test set intact. We employ a fully connected neural network with 2 hidden layers, each containing 512 units and using the ReLU activation function, of the classification task. We use mini-batch SGD with batch size 256 to train this network.

Figure \ref{fig:mnist2_3plots}, Figure \ref{fig:mnist2_ratio} and right panel of Figure \ref{fig:histogram_at-initialization} are based on the above setting.

\paragraph{Multi-class classification on MNIST.} We use the following CNN to classify the $10$ classes of MNIST. Specifically, this CNN contains two consecutive convolutional layers, with $32$ and $64$ channels, respectively. Both convolutional layers uses $3\times 3$ kernel size and are with stride $1$. On top of the convolutional layers, there is one max pooling layer, followed by two fully connected layers with width 64 and 10, respectively. 

We injected $30\%$ random label noise into each class of MNIST training set.
We use mini-batch SGD with batch size 512 to training the neural network.

Top row of Figure \ref{fig:multi-class}, left panel of Figure \ref{fig:mnist10_ratio}, and Figure \ref{fig:1vsrest} are based on this setting.

% {\bf{CNN Network.}} Our CNN consists of 2 consecutive Conv2d layers, from PyTorch\cite{paszke2019pytorch}, with width 32 and 64 respectively. Kernel size 3, and stride 1 was used for both layers. These two layers followed by one max pooling layer and two fully connected layers. We used ReLU activation function after each Conv2d  layers and the first fully connected layer. 
\paragraph{Multi-class classification on CIFAR-10.} For the CIFAR-10 dataset, we use a standard $9$-layer ResNet (ResNet-9) to classify
\footnote{For the detailed architecture, we use the implementation in \url{https://github.com/cbenitez81/Resnet9/blob/main/model_rn.py}.}. 
We injected $40\%$ random label noise into each class of CIFAR-10 training set.
We use mini-batch SGD with batch size 512 to training the neural network.

Bottom row of Figure \ref{fig:multi-class} and right panel of Figure \ref{fig:mnist10_ratio} are based on this setting.

\section{Additional experimental results}
\label{sec:app_verification}
\subsection{Angle $\theta_g$ distribution for multi-class classification}
We experimentally verify the directional distributions of single-logit sample-wise gradients on MNIST dataset. We use the same CNN as in Figure \ref{fig:multi-class}, and evaluate the angle $\theta_g$ distributions at the network initialization. Specifically, given $c\in \{0,1,\cdots,9\}$, we consider the $c$-th output logit. Note that, according to the one-hot encoding, only class $c$ has label $1$ and all the rest classes have label $0$ on this logit. Hence, the binary classifier at logit $c$ is essentially a one-versus-rest classifier. For each of these binary classifiers, we look at the angle $\theta_g$ distributions of the corresponding single-logit sample-wise gradients.

As shown in Figure \ref{fig:1vsrest}, each sub-plot corresponds to one  logit. We can see that, for each $c$:
\begin{itemize}
    \item The clean subset of  class $c$ (green) has its single-logit sample-wise gradients concentrated at small angles $\theta_g$.
    \item The noisy subset of class $c$ (red) has its single-logit sample-wise gradients in the opposite direction of clean ones, concentrating at large angles and being symmetric to the clean subset. The noisy subset gradient $g^{(c)}_{noisy}(\rvw_0)$ (red dash line) is sharply opposite to $g^{(c)}_{clean}(\rvw_0)$, with $\theta_g$ almost $180^{\circ}$.
    \item The distribution of ``other'' subset (blue), which contains the clean samples of all other classes, is clearly separated from the class $c$ distributions. Moreover,  the component of subset gradient $g^{(c)}_{other}(\rvw_0)$ that is orthogonal to $g^{(c)}_{clean}(\rvw_0)$ clearly has non-trivial magnitude (as the $\sin\theta_g \sim \Theta(1)$).
\end{itemize}

All the above observation are align with our analysis for binary classifiers in Section \ref{sec:intialization} (compare with Figure \ref{fig:histogram_at-initialization} for example).

\begin{figure}[h]
        \centering        
        % \hspace{-0.3cm}
        % \vspace{-10pt}
        \includegraphics[width=0.85\textwidth]
        {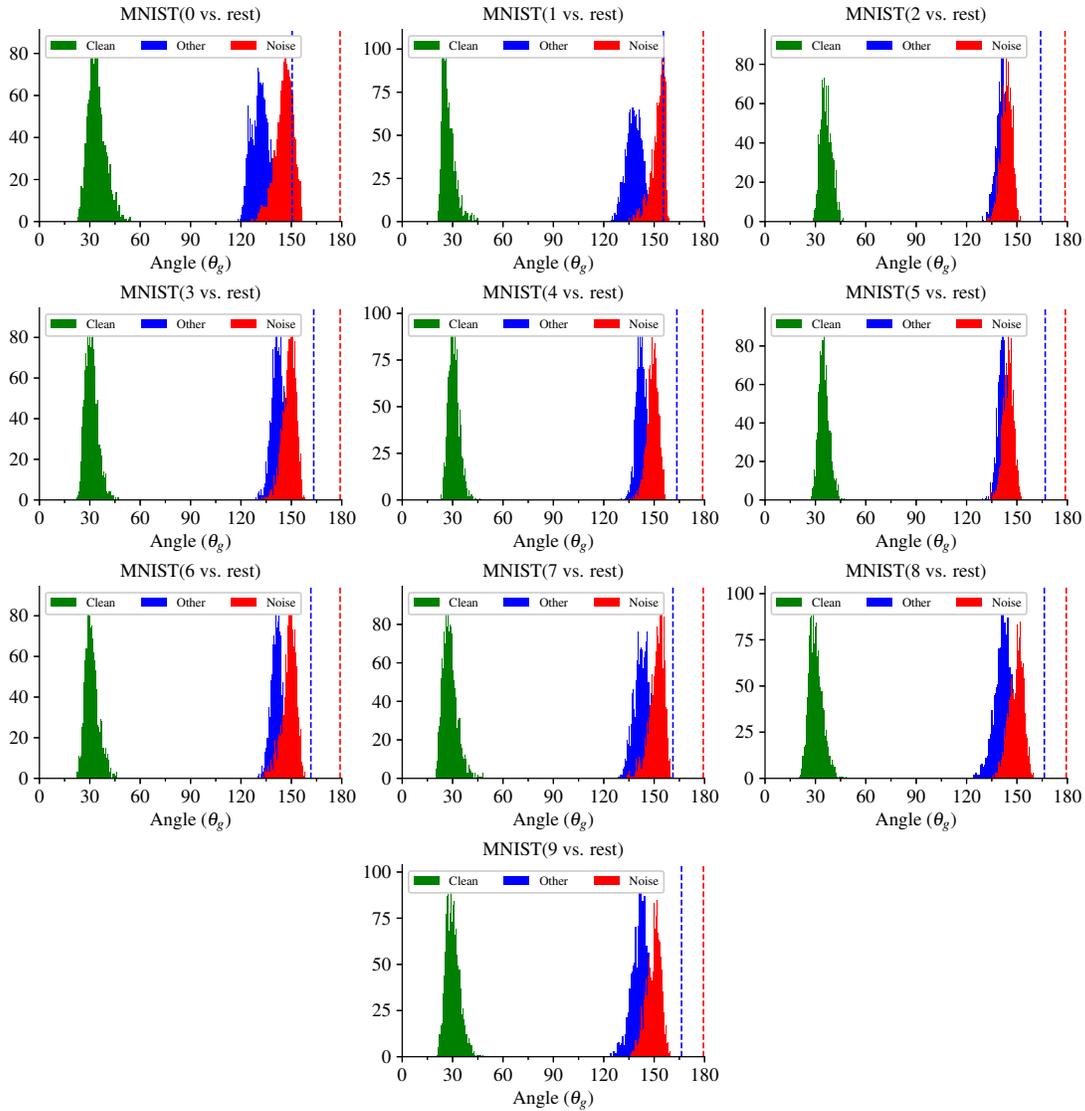}%
        \caption{ The distributions of $\theta_g$ for single-logit sample-wise gradients. MNIST dataset (label noise $\delta=0.3$) on CNN. Dash lines represent subset gradients. 
        }
\vspace{-10pt}
    \label{fig:1vsrest}
\end{figure}

\begin{figure*}[h]
        \centering        
        \includegraphics[width=1.0\columnwidth]
        {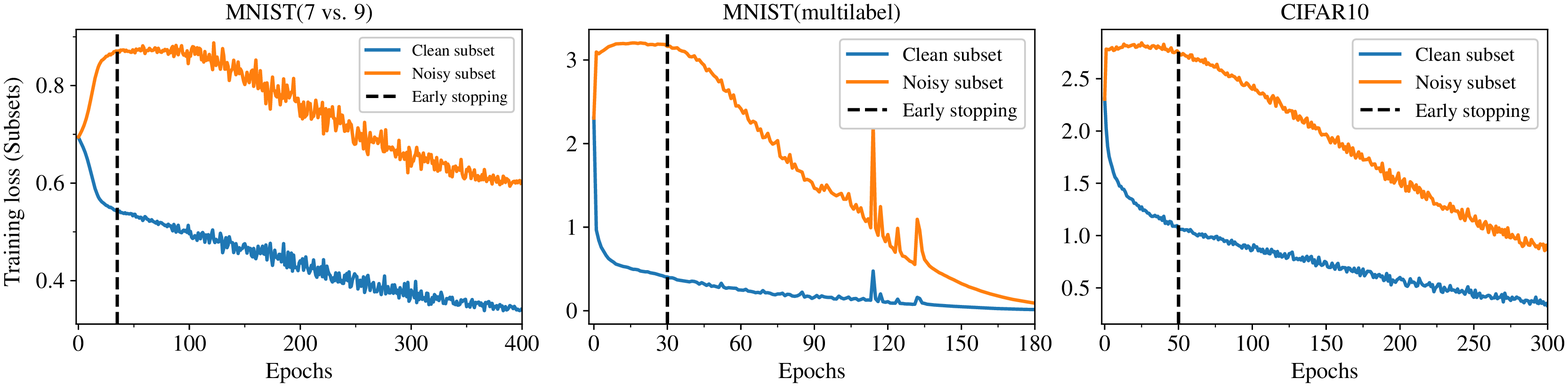}%
        \vspace{-10pt}
        \caption{ Training losses on clean and noisy subsets. {\bf Left}: for binary classification on two class MNIST (``7'' and ``9'', noise level  $\delta=0.4$). {\bf Middle}: for multi-class classification on MNIST (noise level  $\delta=0.3$). {\bf Right}: for multi-class classification on CIFAR-10 (noise level  $\delta=0.4$). 
        }
        \vspace{-10pt}
    \label{fig:loss_flip_clean}
\end{figure*}

\subsection{Subset loss dynamics}
Here, we show the dynamics for the subset losses, i.e., clean subset loss $L(\rvw;\mathcal{D}_{clean})$ and noisy subset loss $L(\rvw;\mathcal{D}_{noisy})$. Figure \ref{fig:loss_flip_clean} shows the curves of these subset losses under different experimental settings: binary classification (same setting as in Figure \ref{fig:mnist2_3plots}); multi-classification for MNIST dataset (same setting as in top row of Figure \ref{fig:multi-class}); and multi-classification for CIFAR-10 dataset (same setting as in bottom row of Figure \ref{fig:multi-class}).

Obviously, under each experimental setting,  the noisy subset loss $L(\rvw;\mathcal{D}_{noisy})$ becomes worse (increases) in the early stage and decreases in the later stage, which is align with the  clean-priority learning dynamics.

\subsection{Effect of width}

We extract two classes, the images with digits ``7'' and ``9'', out from the MNIST datasets, and injected $40\%$ random label noise into each class in the training dataset (i.e., labels of $40\%$ randomly selected samples are flipped to the other class), leaving test set intact. We employ a fully connected neural network with 2 hidden layers. We sweep the number of neuran per layer from 32 to 2048 with ReLU activation function, of the classification task. We use mini-batch SGD with batch size 256 to train this network.
It can be see in Figure \ref{fig:width} that the clean-priority learning is consistent with all widths.

\begin{figure*}[h]
        \centering        
        \includegraphics[width=1.0\columnwidth]
        {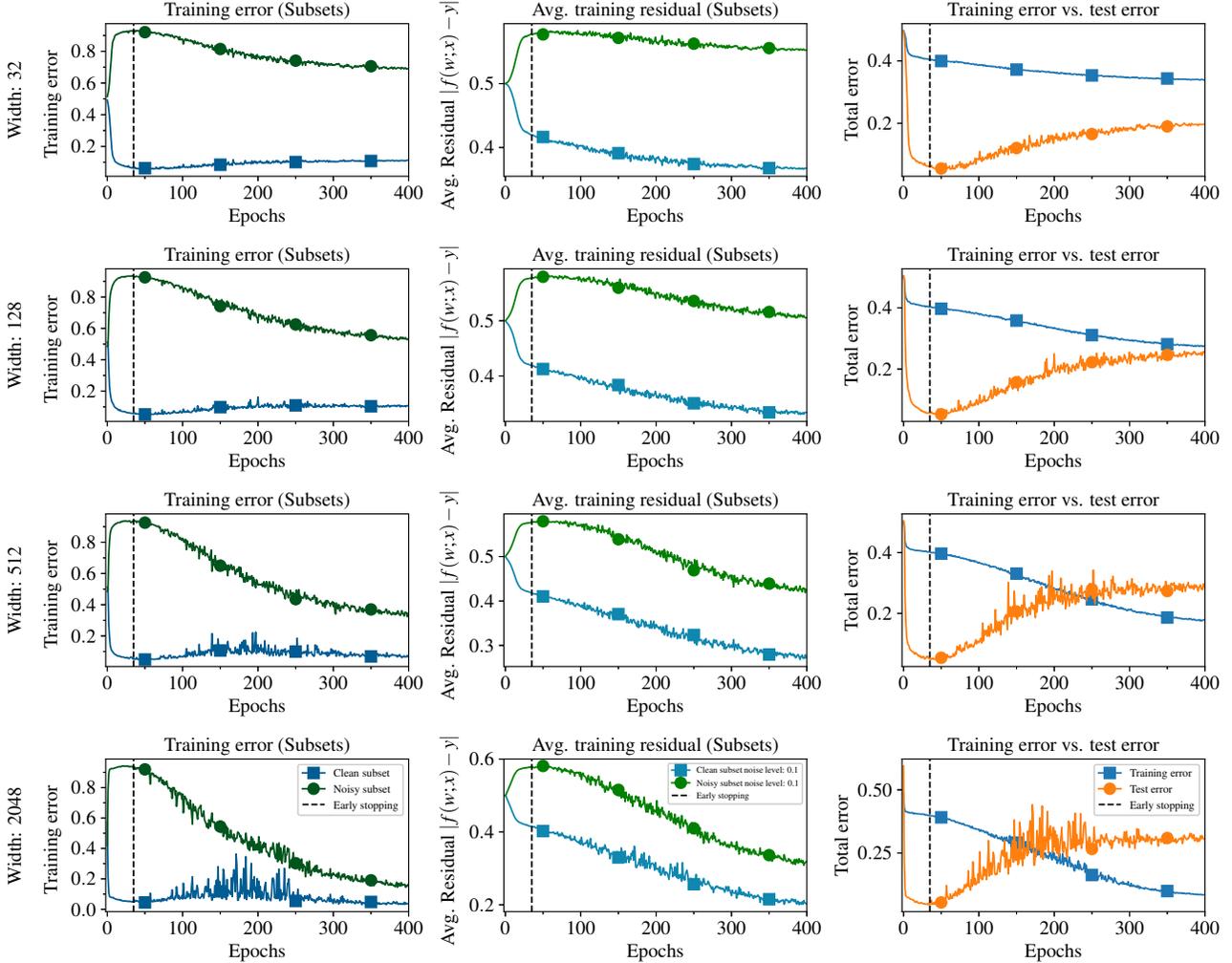}%
        \vspace{-10pt}
        \caption{ Learning dynamics on two classes (``7'' and ``9'') of MNIST (noise level $\delta=0.4$) with FC networks with different widths. {\bf Left}: in the early stage (before the vertical dash line), clean subset error decreases, while  noisy subset error increases. {\bf Middle}: In the early stage, the clean subset average residual $\mathbb{E}_{(\rvx,y)\in \mathcal{D}_{clean}}[|f(\rvw;\rvx)-y|]$ decreases, i.e., on average the network outputs of clean subset move towards the labels, indicating a ``learning'' on the clean subset. One the other hand, the noisy subset average residual, $\mathbb{E}_{(\rvx,y)\in \mathcal{D}_{noise}}[|f(\rvw;\rvx)-y|]$, monotonically increases, indicating that the noisy subset is not-learned. {\bf Right}: total test error and total training error. 
        }
        \vspace{-10pt}d
    \label{fig:width}
\end{figure*}

\subsection{Effect of noise level}

We use the following CNN to classify the $10$ classes of MNIST. Specifically, this CNN contains two consecutive convolutional layers, with $32$ and $64$ channels, respectively. Both convolutional layers uses $3\times 3$ kernel size and are with stride $1$. On top of the convolutional layers, there is one max pooling layer, followed by two fully connected layers with width 64 and 10, respectively. 

We injected different level of random label noise from 0.1 to 0.4 into each class of MNIST training set.
We use mini-batch SGD with batch size 512 to training the neural network.
Figure \ref{fig:noise} shows that clean priority is consitent for all noise levels.

\begin{figure*}[h]
        \centering        
        \includegraphics[width=1.0\columnwidth]
        {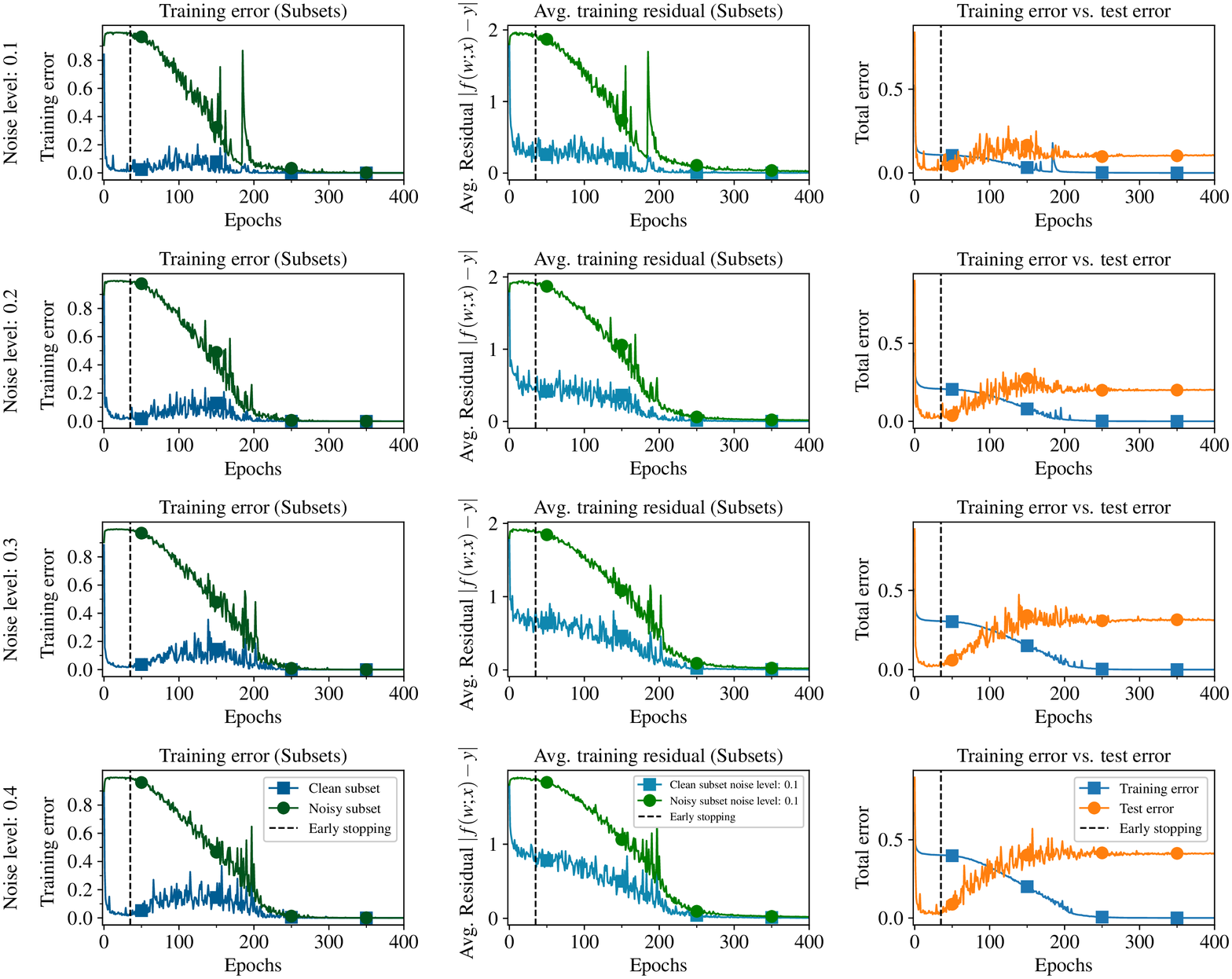}%
        \vspace{-10pt}
        \caption{ Learning dynamics on multi-class classification. {\bf Left}: in the early stage (before the vertical dash line), clean subset error decreases, while  noisy subset error increases. {\bf Middle}: In the early stage, the clean subset average residual $\mathbb{E}_{(\rvx,y)\in \mathcal{D}_{clean}}[\|f(\rvw;\rvx)-y\|]$ decreases, i.e., on average the network outputs of clean subset move towards the labels, indicating a ``learning'' on the clean subset. One the other hand, the noisy subset average residual, $\mathbb{E}_{(\rvx,y)\in \mathcal{D}_{noise}}[\|f(\rvw;\rvx)-y\|]$, monotonically increases, indicating that the noisy subset is not-learned. {\bf Right}: total test error and total training error. See subset loss curves in Appendix \ref{sec:app_verification}.
        }
        \vspace{-10pt}
    \label{fig:noise}
\end{figure*}

%%%%%%%%%%%%%%%%%%%%%%%%%%%%%%%%%%%%%%%%%%%%%%%%%%%%%%%%%%%%%%%%%%%%%%%%%%%%%%%
%%%%%%%%%%%%%%%%%%%%%%%%%%%%%%%%%%%%%%%%%%%%%%%%%%%%%%%%%%%%%%%%%%%%%%%%%%%%%%%

\end{document}